\documentclass[twoside,11pt]{article}

\usepackage{jmlr2e}
\usepackage{amsmath, amssymb, multirow, paralist}
\usepackage{algorithm,algorithmic}
\usepackage{color}

\ShortHeadings{Online Convex Optimization with  Long Term Constraints}{Mahdavi, Jin and Yang}

\newtheorem{assumption}{Assumption}

\def \fh {\hat{f}}
\def \B {\mathcal{B}}
\def \x {\mathbf{x}}

\def \v {\mathbf{v}}
\def \H {\mathcal{H}_{\kappa}}
\def \R {\mathbb{R}}
\def \w {\mathbf{w}}

\def \a {\mathbf{a}}
\def \u {\mathbf{u}}

\def \E {\mathrm{E}}

\def \L {\mathcal{L}}
\def \H {\mathcal{H}}

\def \F {\mathcal{F}}
\def \z {\mathbf{z}}
\def \b {\mathbf{b}}
\def \K {\mathcal{K}}

\def \E {\mathbb{E}}
\def \O {O}

\def \BD {\mathsf{B}}
\def \blambda {\boldsymbol{\lambda}}
\def \bmu {\boldsymbol{\mu}}

\begin{document}

\title{Trading Regret for Efficiency: \\ Online Convex Optimization with  Long Term Constraints}

\author{\name Mehrdad Mahdavi \email mahdavim@cse.msu.edu \\
       \name Rong Jin \email rongjin@cse.msu.edu \\
       \name Tianbao Yang \email yangtia1@msu.edu \\
       \addr Department of Computer Science and Engineering\\
       Michigan State University\\
       East Lansing, MI, 48824, USA
}

\editor{}

\maketitle

\begin{abstract}
In this paper we propose efficient algorithms for solving constrained online convex optimization problems. Our motivation stems from the observation that most  algorithms proposed for online convex optimization require a projection onto the convex set $\mathcal{K}$ from which the decisions are made. While  the projection is straightforward for simple shapes (e.g., Euclidean ball), for arbitrary complex sets it is  the main computational challenge and may be inefficient in practice. In this paper, we consider an alternative online convex optimization problem. Instead of requiring that decisions belong to $\mathcal{K}$  for all rounds, we only require that the constraints, which define the set $\mathcal{K}$, be satisfied in the long run.
By turning the problem into an online convex-concave optimization problem, we propose an efficient algorithm which achieves $O(\sqrt{T})$  regret bound and $O(T^{3/4})$ bound on the violation of constraints. Then, we modify the algorithm in order to guarantee that the constraints are satisfied in the long run. This gain is achieved at the price of getting $O(T^{3/4})$  regret bound. Our second algorithm is based on the mirror prox method \citep{nemirovski-2005-prox} to solve variational inequalities which achieves $O(T^{2/3})$ bound for both regret and the violation of constraints  when the domain $\K$ can be described by a finite number of linear constraints. Finally, we extend the results to the setting where we only have partial access to the convex set $\mathcal{K}$ and propose a multipoint bandit feedback algorithm with the same bounds in expectation as our first algorithm.
\end{abstract}


\section{Introduction}
Online convex optimization has recently emerged as a primitive framework for designing efficient algorithms for a wide variety of machine learning applications \citep{Cesa-Bianchi:2006:PLG:1137817}. In general, an online convex optimization problem  can be formulated as a repeated game between a learner and an adversary:  at each iteration $t$, the learner first presents a solution $\x_t \in \K$, where $\K \subseteq \R^d$ is a convex domain representing the solution space; it then receives a convex function $f_t(\x): \K \mapsto \R_+$ and suffers the loss $f_t(\x_t)$ for the submitted solution $\x_t$. The objective of the learner is to generate a sequence of solutions $\x_t \in \K, t = 1, 2, \cdots, T$ that minimizes the regret $\mathfrak{R}_T$ defined as

\begin{eqnarray}
\label{eqn:regret}
\mathfrak{R}_T = \sum_{t=1}^T f_t(\x_t) - \min\limits_{\x \in \K} \sum_{t=1}^T f_t(\x).
\end{eqnarray}
Regret measures the difference between the cumulative loss of the learner's strategy and the minimum possible loss had the sequence of loss functions been known in advance and the learner could choose the best fixed action in hindsight. When $\mathfrak{R}_T$ is sub-linear in the number of rounds, that is,  $o(T)$, we call the solution Hannan consistent \citep{Cesa-Bianchi:2006:PLG:1137817}, implying that the learner's average per-round loss approaches the average per-round loss of the best fixed action in hindsight. It is noticeable that  the performance bound must hold for any sequence of loss functions, and in particular  if the sequence is chosen adversarially.

Many successful algorithms have been developed over the past decade to minimize the regret in the online convex optimization. The problem was initiated in the remarkable  work of \citet{DBLP:conf/icml/Zinkevich03} which presents an algorithm  based on gradient descent with projection that guarantees a regret of $\O(\sqrt{T})$ when the set $\K$ is convex  and the loss functions are  Lipschitz continuous within the domain $\K$. In \citet{Hazan:2007:LRA:1296038.1296051} and \citet{DBLP:conf/nips/Shalev-ShwartzK08} algorithms with logarithmic regret bound were proposed for strongly convex loss functions. In particular, the algorithm in \citet{Hazan:2007:LRA:1296038.1296051} is based on online Newton step and covers the  general class of exp-concave loss functions.  Notably, the simple gradient based algorithm also achieves an $\O(\log T)$ regret bound for strongly convex loss functions with an appropriately chosen step size. \citet{DBLP:conf/nips/BartlettHR07} generalizes the results in previous works to the setting where the algorithm can adapt to the curvature  of the loss functions without any prior information. A modern view of these algorithms  casts the problem as the task of following the regularized leader \citep{lecture-notes}. In \citet{DBLP:conf/colt/AbernethyABR09}, using game-theoretic analysis, it has been shown that   both $\O(\sqrt{T})$ for Lipschitz continuous  and $\O(\log T)$ for strongly convex loss functions are tight in the minimax sense. 

Examining the existing algorithms, most of the techniques usually require a projection step at each iteration in order to get back to the feasible region.  For the performance of these online algorithms, the computational cost of the projection step is of crucial importance. To motivate the setting addressed in this paper, let us first examine a popular online learning algorithm for minimizing the regret $\mathfrak{R}_T$ based on the online gradient descent (OGD) method \citep{DBLP:conf/icml/Zinkevich03}. At each iteration $t$, after receiving the convex function $f_t(\x)$, the learner computes the gradient $\nabla f_t(\x_t)$ and updates the solution $\x_t$ by solving the following optimization problem
\begin{eqnarray}
\label{eqn:proj}
    \x_{t+1} = \Pi_{\K}\left(\x_t - \eta \nabla f_t(\x_t)\right) = \mathop{\arg\min}\limits_{\x \in \K} \left \|\x - \x_t + \eta \nabla f_t(\x_t)\right\|^2,
\end{eqnarray}
where  $ \Pi_{\K}(\cdot)$ denotes the projection onto  $\K$ and $\eta > 0$ is a predefined step size. Despite the simplicity of the OGD algorithm, the computational cost per iteration is crucial for its applicability.  For general convex domains, solving the optimization problem in (\ref{eqn:proj}) is an offline convex optimization problem by itself and can be computationally expensive. For example, when one envisions a positive semidefinitive cone in  applications such as distance metric learning and matrix completion, the full eigen-decomposition of a matrix is required to project the updated solutions back into the cone. Recently several efficient algorithms have been developed for projection onto specific domains, for example,  $\ell_1$ ball \citep{DBLP:conf/icml/DuchiSSC08, DBLP:conf/icml/LiuY09}; however, when the domain $\K$ is complex, the projection step  is a more involved task or  computationally burdensome.

To tackle the computational challenge arising from the projection step, we consider an alternative online learning problem. Instead of requiring $\x_t \in \K$, we only require the constraints, which define the convex domain $\K$, to be satisfied in a long run. Then, the online learning problem becomes a task to find a sequence of solutions $\x_t, t \in [T]$ that minimizes the regret defined in (\ref{eqn:regret}), under the  long term constraints, that is, $\sum_{t=1}^T \x_t/T \in \K$. We refer to this problem as {\bf online learning with long term constraints}. In other words, instead of solving the projection problem in (\ref{eqn:proj}) on each round, we allow the learner to make decisions at some iterations which do not belong to the set $\K$, but the overall sequence of chosen decisions  must obey the constraints at the end by a vanishing convergence rate.

From a different perspective, the proposed online optimization with long term constraints  setup is reminiscent of regret minimization with side constraints or constrained regret minimization addressed in \citet{DBLP:conf/colt/MannorT06}, motivated by applications in wireless communication. In regret minimization with side constraints, beyond minimizing regret, the learner has some side constraints that need to be satisfied on average for all rounds. Unlike our setting, in learning with side constraints, the set $\K$ is controlled by adversary and can vary arbitrarily from trial to trial. It has been shown that if the convex set is affected by both decisions and loss functions, the minimax optimal regret is generally unattainable online \citep{DBLP:journals/jmlr/MannorTY09}. 

One interesting application of  the constrained regret minimization is multi-objective online classification   where the learner aims at simultaneously optimizing more than one classification performance criteria. In the simple two objective online classification considered in \citet{DBLP:conf/nips/BernsteinMS10}, the goal of the online classifier is to maximize the average true positive classification rate with an additional performance guarantee in terms of the false positive rate. Following the Neyman-Pearson risk,  the intuitive approach to tackle this problem is to optimize one criterion (i.e., maximizing the true positive rate) subject to explicit constraint on the other objective (i.e., false positive rate) that needs to be satisfied on average for the sequence of  decisions. The constrained regret matching (CRM) algorithm, proposed in \citet{DBLP:conf/nips/BernsteinMS10},  efficiently solves this problem by relaxing the objective under mild assumptions on the single-stage constraint.  The main idea of the CRM algorithm is  to incorporate the penalty, that should be paid by the learner to satisfy the constraint,  in the objective (i.e., true positive rate) by subtracting a positive constant at each decision step.  It has been shown that the CRM algorithm \textit{asymptotically} satisfies the average constraint (i.e., false positive rate) provided that the relaxation constant is above a certain threshold.

Finally, it is worth mentioning that the proposed setting  can be used in certain classes of online learning such as  online-to-batch conversion \citep{DBLP:journals/tit/Cesa-BianchiCG04}, where it is sufficient to guarantee that  constraints are  satisfied in the long run. More specifically, under the assumption that received  examples are i.i.d samples, the solution for batch learning is to average the solutions obtained over all the trials. As a result, if the long term constraint is satisfied, it  is guaranteed that the average solution will belong to the domain $\K$.

In this paper, we describe and analyze a general framework for solving online convex optimization with long term constraints.  We first show that a direct application of OGD fails to achieve a sub-linear bound on the violation of constraints  and an $\O(\sqrt{T})$ bound on the regret. Then, by turning the problem into an online convex-concave optimization problem, we propose an efficient algorithm which is an adaption of OGD for online learning with long term constraints. The proposed algorithm achieves the same $\O(\sqrt{T})$ regret bound as the general setting  and $\O(T^{3/4})$ bound for the violation of constraints. We show that by using a simple trick we can turn the  proposed method into an algorithm which exactly satisfies the constraints in the long run by achieving $\O(T^{3/4})$ regret bound.  When the convex domain $\K$ can be described by a finite number of linear constraints, we propose an alternative algorithm based on the mirror prox method \citep{nemirovski-2005-prox}, which achieves $\O(T^{2/3})$ bound for both regret and the violation of constraints. Our  framework also handles the cases when we do not have  full access to the domain $\K$ except through a limited number of oracle evaluations.  In the full-information version, the decision maker can observe the entire convex domain $\K$, whereas in a partial-information (a.k.a bandit setting) the
decision maker may only observe the cost of the constraints defining the domain $\K$ at limited points. We show that we can generalize  the proposed OGD based algorithm to this setting by only accessing the value oracle for domain $\K$ at two points, which achieves  the same bounds  in expectation  as the case that has a full knowledge about the domain $\K$. In summary, the present work makes the following contributions:
\begin{itemize}
\item  A general theorem that shows, in online setting, a simple penalty based method attains linear bound $O(T)$ for either the regret or the long term violation of the constraints and fails to achieve sub-linear bound for both regret and the long term violation of the constraints at the same time.
\item A convex-concave formulation of online convex optimization with long term constraints, and an efficient algorithm based on OGD that attains a regret bound of  $O(T^{1/2})$, and $O(T^{3/4})$ violation of the constraints.
\item A modified OGD based algorithm for online convex optimization with long term constraints that has no constraint violation but $O(T^{3/4})$ regret bound.
\item An algorithm for online convex optimization with long term constraints based on the mirror prox method that achieves $O(T^{2/3})$ regret  and constraint violation.
\item A multipoint bandit version of the basic algorithm with $O(T^{1/2})$ regret bound and $O(T^{3/4})$ violation of the constraints in \textit{expectation} by accessing the value oracle for the convex set $\K$ at two points.
\end{itemize}
The remainder of the paper is structured as follows: In Section \ref{sec:ogd}, we first examine a simple penalty based strategy and show that it fails to attain sub-linear bound for both regret and long term violation of the constraints. Then,  we formulate regret minimization as an online convex-concave optimization problem and apply the OGD algorithm to solve it.  Our first algorithm allows the constraints to be violated in a controlled way. It is then modified to have the constraints exactly satisfied in the long run. Section \ref{sec:mirror} presents our second algorithm which is an adaptation of the mirror prox method. Section \ref{sec:bandit} generalizes the online convex optimization with long term constraints problem to the setting where we only have a partial access to the convex domain $\K$. Section \ref{sec:con} concludes the work with a list of open questions.
\section {Notation and Setting} 
Before proceeding, we define the notations used throughout the paper and state the assumptions made for the analysis of algorithms. Vectors are shown by lower case
bold letters, such as $\x \in \R^d$. Matrices are indicated by upper case letters such as $A$ and their pseudoinverse is represented by $A^{\dagger}$. We use $[m]$ as a shorthand for the set of integers $\{1, 2, \ldots, m \}$. Throughout the paper we denote by $\|\cdot\|$ and $\|\cdot\|_1$ the $\ell_2$ (Euclidean) norm and $\ell_1$-norm, respectively. We use $\E$ and $\E_t$ to denote the expectation and conditional expectation with respect to all randomness in early $t-1$ trials, respectively. To facilitate our analysis, we assume that the domain $\K$ can be written as an intersection of a finite number of  convex constraints, that is,
\[ \K = \{\x \in \R^d: g_i(\x) \leq 0, i \in [m]\}, \]
where $g_i(\cdot), i \in [m]$,  are Lipschitz continuous functions. Like many other  works for online convex optimization such as \citet{DBLP:conf/soda/FlaxmanKM05}, we assume that $\K$  is a bounded domain, that is, there exist  constants $R > 0$ and $r < 1$ such that $\K \subseteq R \mathbb{B}$  and $r \mathbb{B} \subseteq \K$ where $\mathbb{B}$ denotes the unit $\ell_2$ ball centered at the origin.  For the ease of notation, we use $\B = R \mathbb{B}$.

We focus on the problem of online convex optimization, in which the goal is to achieve a low regret with respect to a fixed decision on a sequence of loss functions. The difference between the setting considered here and the general online convex optimization is that, in our setting, instead of requiring $\x_t \in \K$, or equivalently $g_i(\x_t) \leq 0, i \in [m]$, for all $t \in [T]$, we only require the constraints to be satisfied in the long run, namely $\sum_{t=1}^T g_i(\x_t) \leq 0, i \in [m]$.
Then, the problem becomes to find a sequence of solutions $\x_t, t \in [T]$ that minimizes the regret defined in (\ref{eqn:regret}), under the  long term constraints $\sum_{t=1}^T g_i(\x_t) \leq 0, i \in [m]$. Formally, we would like to solve the following optimization problem online,
\begin{eqnarray}
\label{eqn:setting1}
\min_{\x_1, \ldots, \x_T \in \B} \sum_{t=1}^T f_t(\x_t) - \min\limits_{\x \in \K} \sum_{t=1}^T f_t(\x) \;\; \text{s.t.} \;\;\sum_{t = 1}^{T}{g_i(\x_t)} \leq 0\; \text{, } i \in [m].
\end{eqnarray}

For simplicity, we will focus on a finite-horizon setting where the number of rounds $T$ is known in advance. This condition can be relaxed under certain conditions, using standard techniques (see, e.g., \citealp{Cesa-Bianchi:2006:PLG:1137817}). Note that in (\ref{eqn:setting1}), (i) the solutions come from the ball $\B \supseteq \K$ instead of $\K$  and (ii) the constraint functions are fixed and are given in advance.

Like most online learning algorithms, we assume that both loss functions and the constraint functions are Lipschitz continuous, that is, there exists constants $L_f$ and $L_g$ such that
$|f_t(\x) - f_t(\x')| \leq L_f \|\x - \x'\|,\;\;\; |g_i(\x) - g_i(\x')| \leq L_g \|\x - \x'\| \; \; \text{for any} \; \x \in \B\; \text{and}  \; \x' \in \B\;, i \in [m]. $
For  simplicity of analysis, we use $G = \max \{L_f, L_g \}$ and 
\begin{eqnarray*}
F & = & \max\limits_{t \in [T]} \max\limits_{\x, \x' \in \K} f_t(\x) - f_t(\x') \leq 2 L_f R,   \\
D & = & \max\limits_{i \in [m]} \max\limits_{\x \in \B} g_i(\x) \leq L_g R.
\end{eqnarray*}

Finally, we define the notion of a Bregman divergence. Let $\phi(\cdot)$ be a strictly convex function defined on a convex set $\K$. The Bregman divergence between $\x$ and $\x'$ is defined as $\BD_{\phi}(\x, \x') = \phi(\x) - \phi(\x') -  (\x - \x')^{\top} \nabla \phi(\x')$ which measures how much the function $\phi(\cdot)$ deviates at $\x$ from it's linear approximation at $\x'$.

\section{Online Convex Optimization with Long Term Constraints}
\label{sec:ogd}
In this section we present and analyze our gradient descent based algorithms for online convex optimization problem with long term constraints. We first describe an algorithm which is allowed to violate the constraints and then, by applying a simple trick, we propose a variant of the first  algorithm which exactly satisfies the constraints in the long run. 

Before we state our formulation and algorithms, let us review a few alternative techniques that do not need explicit projection. A straightforward approach is  to introduce an appropriate self-concordant barrier function for the given convex set $\K$ and add it to the objective function such that the barrier diverges at the boundary of the set. Then we can  interpret the resulting optimization problem, on the modified objective functions, as an unconstrained minimization problem that can  be solved without projection steps. Following the analysis in \citet{DBLP:journals/tit/AbernethyHR12}, with an appropriately designed procedure for  updating solutions, we could guarantee a regret bound of $O(\sqrt{T})$ without the violation of constraints. A similar idea is used in \citet{DBLP:conf/colt/AbernethyHR08} for online bandit learning and in \citet{DBLP:conf/nips/NarayananR10} for a random walk approach for regret minimization which, in fact,  translates the issue of projection into the difficulty of sampling. Even for linear Lipschitz cost functions, the random walk approach requires sampling from a Gaussian distribution with covariance given by the Hessian of the self-concordant barrier of the convex set $\K$ that has the same  time complexity as inverting a  matrix. The main limitation with these approaches is that they require computing the Hessian matrix of the objective function in order to guarantee that the updated solution stays within the given domain $\K$. This limitation makes it computationally unattractive when dealing with high dimensional data. In addition, except for  well known cases, it is often unclear  how to efficiently construct a self-concordant barrier function for a general convex domain.

An alternative approach for online convex optimization with long term constraints is to introduce a penalty term in the loss function that penalizes the violation of constraints. More specifically, we can define a new loss function $\fh_t(\cdot)$ as
\begin{eqnarray}
\label{eqn:fhat}
    \fh_t(\x) = f_t(\x) + \delta\sum_{i=1}^{m} [g_i(\x)]_+,
\end{eqnarray}
where $[z]_+ = \max(0, 1 - z)$  and $\delta > 0$ is a fixed positive constant used to penalize the violation of constraints. We then run the standard OGD algorithm to minimize the modified loss function $\fh_t(\cdot)$. The following theorem shows that this simple strategy fails to achieve sub-linear bound for both  regret and the long term violation of constraints at the same time.

\begin{theorem} \label{thm:a}
Given $\delta > 0$, there always exists a sequence of loss functions $\{f_t(\x)\}_{t=1}^T$ and a constraint function $g(\x)$ such that either $\sum_{t=1}^T f_t(\x_t) - \min_{g(\x) \leq 0} \sum_{t=1}^T f_t(\x) = O(T)$ or $\sum_{t=1}^T [g(\x_t)]_+ = O(T)$ holds, where $\{\x_t\}_{t=1}^T$ is the sequence of solutions generated by the OGD algorithm that minimizes the modified loss functions given in (\ref{eqn:fhat}).
\end{theorem}

We defer the proof to Appendix \ref{app:thm:a} along with a simple analysis of the OGD when applied to the modified functions in (\ref{eqn:fhat}). The analysis shows that in order to obtain $O(\sqrt{T})$ regret bound, linear bound on the long term violation of the constraints is unavoidable. The main reason for the failure of using  modified loss function in (\ref{eqn:fhat}) is that the weight constant $\delta$ is fixed and independent from the sequence of solutions obtained so far. In the next subsection, we present an online 
convex-concave formulation for online convex optimization with long term constraints, which explicitly addresses the limitation of (\ref{eqn:fhat}) by automatically adjusting the weight constant based on the violation of the solutions obtained so far.

As mentioned before, our general strategy is to turn online convex optimization with long term constraints into a convex-concave optimization problem. Instead of generating a sequence of solutions that satisfies the long term constraints, we first consider an online optimization strategy that allows  the violation of constraints on some rounds in a controlled way. We then modify the online optimization strategy to obtain a sequence of solutions that obeys the long term constraints. Although the online convex optimization with long term constraints is clearly easier than the standard online convex optimization problem, it is straightforward to see that optimal regret bound for online optimization with long term constraints should be on the order of $\O(\sqrt{T})$, no better than the standard online convex optimization problem.

\subsection {An Efficient Algorithm with $\O(\sqrt{T})$ Regret Bound and $\O(T^{3/4})$ Bound on the Violation of Constraints}
The intuition behind  our approach stems from the observation that the constrained optimization problem $\min_{\x \in \K} \sum_{t=1}^T f_t(\x)$ is equivalent to the following convex-concave optimization problem
\begin{eqnarray}
\label{eqn:saddle}
    \min\limits_{\x \in \B} \max\limits_{\lambda \in \R_+^m} \sum_{t=1}^T f_t(\x) + \sum_{i=1}^m \lambda_i g_i(\x), \label{eqn:convex-concave}
\end{eqnarray}
where  $\blambda = (\lambda_1, \ldots, \lambda_m)^{\top}$ is the vector of Lagrangian multipliers  associated with the constraints $g_i(\cdot), i = 1, \ldots, m$ and belongs to the nonnegative orthant $\R_+^m$.  To solve the online convex-concave optimization problem, we extend the gradient based approach for variational inequality  \citep{nemirovski-efficient} to (\ref{eqn:convex-concave}). To this end, we consider the following regularized convex-concave function  as
\begin{eqnarray}
\label{eqn:cox-con-func}
    \L_t(\x, \blambda) = f_t(\x) + \sum_{i=1}^m \left\{\lambda_i g_i(\x) - \frac{\delta \eta}{2} \lambda_i^2\right\}, \label{eqn:L}
\end{eqnarray}
where $\delta > 0$ is a constant whose value will be decided by the analysis. Note that in (\ref{eqn:L}), we introduce a regularizer $\delta \eta \lambda_i^2/2$ to prevent $\lambda_i$ from being too large. This is because, when $\lambda_i$ is large, we may encounter a large gradient for $\x$ because of $\nabla_{\x}\L_t(\x, \blambda) \propto \sum_{i=1}^m \lambda_i \nabla g_i(\x)$, leading to unstable solutions and a poor regret bound. Although we can achieve the same goal by restricting $\lambda_i$ to a bounded domain, using the quadratic regularizer makes it convenient for our analysis.

Algorithm~\ref{alg:1} shows the detailed steps of the proposed algorithm. Unlike standard online convex optimization algorithms that only update $\x$, Algorithm~\ref{alg:1} updates both $\x$ and $\blambda$. In addition, unlike the modified loss function in (\ref{eqn:fhat}) where the weights for constraints $\{g_i(\x) \leq 0\}_{i=1}^m$ are fixed, Algorithm~\ref{alg:1} automatically adjusts the weights $\{\lambda_i\}_{i=1}^m$ based on $\{g_i(\x)\}_{i=1}^m$, the violation of constraints, as the game proceeds. It is this property that allows Algorithm~\ref{alg:1} to achieve sub-linear bound for both  regret and  the violation of  constraints. 

To analyze Algorithm \ref{alg:1}, we first state the following lemma, the key to the main theorem on the regret bound and the violation of constraints.

\begin{lemma}
\label{lemma:basic-ineq}
Let $\L_t(\cdot, \cdot)$ be the function defined in (\ref{eqn:cox-con-func}) which is convex in its first argument and concave in its second argument. Then for any $(\x, \blambda)  \in \B \times  \R_+^m$ we have
\begin{align*}
  \L_t(\x_t, \blambda) - \L_t(\x, \blambda_t) & \leq  \frac{1}{2\eta} (\|\x - \x_t\|^2 + \|\blambda - \blambda_t\|^2 - \|\x - \x_{t+1}\|^2 - \|\blambda - \blambda_{t+1}\|^2) \\ & \hspace{5mm}+ \frac{\eta}{2} ( \|\nabla_{\x}\L_t(\x_t, \blambda_t)\|^2 + \|\nabla_{\blambda}\L_t(\x_t, \blambda_t)\|^2 ).
\end{align*}
\end{lemma}
\begin{proof}Following the analysis of \citet{DBLP:conf/icml/Zinkevich03}, convexity of $\L_t(\cdot, \blambda)$ implies that
\begin{eqnarray}
\L_t(\x_t, \blambda_t) - \L_t(\x, \blambda_t) \leq (\x_{t} - \x)^{\top}\nabla_\x \L_t(\x_t, \blambda_t)
\label{eqn:conv}
\end{eqnarray}
and by concavity of $\L_t( \x, \cdot)$ we have
\begin{eqnarray}
\L_t(\x_t, \blambda) - \L_t(\x_t, \blambda_t) \leq (\blambda - \blambda_{t})^{\top}\nabla_{\blambda} \L_t(\x_t, \blambda_t).
\label{eqn:conc}
\end{eqnarray}
Combining the inequalities (\ref{eqn:conv}) and (\ref{eqn:conc}) results in
\begin{eqnarray}
\L_t(\x_t, \blambda) - \L_t(\x, \blambda_t) \leq (\x_{t} - \x)^{\top}\nabla_\x \L_t(\x_t, \blambda_t) - (\blambda - \blambda_{t})^{\top}\nabla_{\blambda} \L_t(\x_t, \blambda_t).
\label{eqn:convconc}
\end{eqnarray}
Using the update rule for $\x_{t+1}$ in terms of $\x_t$ and expanding, we get
\begin{eqnarray}
\label{eqn:x-inq}
\|\x-\x_{t+1} \|^2 \leq \|\x-\x_t\|^2 - 2 \eta (\x_{t} - \x)^{\top}\nabla_\x \L_t(\x_t, \blambda_t) + \eta^2 \| \nabla_\x \L_t(\x_t, \blambda_t)\|^2,
\end{eqnarray}
where the first inequality follows from the nonexpansive property of the projection operation.
Expanding the inequality for $\|\blambda - \blambda_{t+1} \|^2$ in terms of $\blambda_t$ and plugging back into the (\ref{eqn:convconc}) with (\ref{eqn:x-inq}) establishes the desired inequality.
\end{proof}
\begin{algorithm}[tb]
\caption{Gradient based  Online Convex Optimization with Long Term Constraints}
\begin{algorithmic}[1] \label{alg:1}
    \STATE {\bf Input}: constraints $g_i(\x) \leq 0, i \in [m]$, \; step size $\eta$, and constant $\delta > 0$
    \STATE {\bf Initialization}: $\x_1 = \mathbf{0}$ and $\blambda_1 = \mathbf{0}$
    \FOR{$t = 1, 2, \ldots, T$}
        \STATE Submit solution $\x_t$
        \STATE Receive the convex function $f_t(\x)$ and experience loss $f_t(\x_t)$
        \STATE Compute $\nabla_\x \L_t(\x_t, \blambda_t)=\nabla f_t(\x_t) + \sum_{i=1}^m \lambda_t^i \nabla g_i(\x_t)$ and $\nabla_{\lambda_i}\L_t(\x_t, \blambda_t) =  g_i(\x_t) - \eta \delta \lambda_t^i$
        \STATE Update $\x_t$ and $\blambda_t$ by
        \begin{eqnarray*}
        && \x_{t+1} = \Pi_{\B}\left(\x_t - \eta \nabla_\x \L_t(\x_t, \blambda_t)\right) \\  
        && \blambda_{t+1} = \Pi_{[0, +\infty)^m}(\blambda_t + \eta \nabla_{\blambda} \L_t(\x_t, \blambda_t))
        \end{eqnarray*}
    \ENDFOR
\end{algorithmic}
\end{algorithm}
\begin{proposition}
\label{prop:2}
 Let $\x_t$ and $\blambda_t, t \in [T]$ be the sequence of solutions obtained by Algorithm~\ref{alg:1}. Then for any $\x \in \B$ and $\blambda  \in \R_+^m$, we have
\begin{align}
\label{eqn:center}
& \lefteqn{\sum_{t=1}^T \L_t(\x_t, \blambda) - \L_t(\x, \blambda_t)}  \\
 & \leq \frac{R^2 + \|\blambda\|^2}{2\eta} + \frac{\eta T}{2 }\left((m+1)G^2 + 2mD^2\right) + \frac{\eta }{2} \left((m +1)G^2 + 2m \delta^2\eta^2\right)\sum_{t=1}^T \|\blambda_t\|^2. \nonumber
\end{align}
\end{proposition}
\begin{proof}
We first bound the gradient terms in the right hand side of  Lemma \ref{lemma:basic-ineq}.  Using the inequality $(a_1+a_2+ \ldots, a_n)^2 \leq n (a_1^2+a_2^2+\ldots+ a_n^2)$, we have $\|\nabla_{\x} \L_t(\x_t, \blambda_t)\|^2 \leq (m + 1) G^2 \left(1 + \|\blambda_t\|^2 \right)$ and $ \|\nabla_{\blambda}\L_t(\x_t, \blambda_t)\|^2 \leq 2m (D^2 + \delta^2\eta^2 \|\blambda_t\|^2)$. In Lemma \ref{lemma:basic-ineq}, by adding the inequalities of all iterations, and using the fact $\|\x\| \leq R$ we complete the proof.
\end{proof}
The following theorem bounds the regret and the violation of the constraints in the long run for Algorithm \ref{alg:1}.
\begin{theorem}
\label{thm:ogd}
 Define $a = R\sqrt{(m+1)G^2 + 2mD^2}$. Set $\eta = R^2/[a\sqrt{T}]$. Assume $T$ is large enough such that $2\sqrt{2}\eta (m+1) \leq 1$. Choose $\delta$ such that $\delta \geq (m +1)G^2 + 2m \delta^2\eta^2$. Let $\x_t, t \in [T]$ be the sequence of solutions obtained by Algorithm~\ref{alg:1}. Then for the optimal solution $\x_* = \min_{\x \in \K} \sum_{t=1}^T f_t(\x)$ we have
\begin{eqnarray*}
&& \sum_{t=1}^T f_t(\x_t) - f_t(\x_*) \leq  a \sqrt{T} = \O(T^{1/2}), \;\text{and}\\  
&& \sum_{t=1}^T g_i(\x_t) \leq \sqrt{2\left(F T + a \sqrt{ T}\right)\sqrt{T}\left(\frac{\delta R^2}{a} + \frac{ma}{R^2}\right)} = \O(T^{3/4}).
\end{eqnarray*}
\end{theorem}
\begin{proof}
We begin by expanding (\ref{eqn:center}) using (\ref{eqn:L}) and rearranging the terms to get 
\begin{align*}
\lefteqn{\sum_{t=1}^T \left[f_t(\x_t) - f_t(\x)\right] + \sum_{i=1}^m \left\{\lambda_i\sum_{t=1}^T g_i(\x_t) - \sum_{t=1}^T \lambda^i_t g_i(\x)\right\} - \frac{\delta\eta T}{2}\|\blambda\|^2} \\
& \leq -\frac{\delta\eta}{2}\sum_{t=1}^T\|\blambda_t\|^2 + \frac{R^2 + \|\blambda\|^2}{2\eta} + \frac{\eta T}{2}\left((m+1)G^2 + 2mD^2\right) \\ &\hspace{5mm}+ \frac{\eta }{2} \left((m +1)G^2 + 2m \delta^2\eta^2\right)\sum_{t=1}^T \|\blambda_t\|^2.
\end{align*}
Since $\delta \geq  (m +1)G^2 + 2m \delta^2\eta^2$, we can drop the $\|\blambda_t\|^2$ terms from both sides of the above inequality and obtain
\begin{eqnarray*}
\lefteqn{\sum_{t=1}^T \left[f_t(\x_t) - f_t(\x)\right] + \sum_{i=1}^m \left\{\lambda_i\sum_{t=1}^T g_i(\x_t) - \left(\frac{\delta\eta T}{2} + \frac{m}{2\eta}\right)\lambda_i^2\right\}}  \;\;\;\;\;\;\;\;\;\;\;\;\;\;\;\;\;\;\;\;\;\;\;\;\\
 & \leq &\sum_{i=1}^m \sum_{t=1}^T \lambda^i_t g_i(\x)  + \frac{R^2}{2\eta} + \frac{\eta T}{2}\left((m+1)G^2 + 2mD^2)\right). \!\!\!\!\!\!\!\!\!\!\!\!
\end{eqnarray*}
The left hand side of above inequality consists of two terms. The first term basically measures the difference between the cumulative loss of the Algorithm \ref{alg:1}  and the optimal solution and the second term includes the constraint functions with corresponding Lagrangian multipliers which will be used to bound the long term violation of the constraints. By taking  maximization for $\blambda$ over the range $(0, +\infty)$, we get
\begin{eqnarray*}
\sum_{t=1}^T \left[f_t(\x_t) - f_t(\x)\right] + \sum_{i=1}^m \left\{\frac{\left[\sum_{t=1}^T g_i(\x_t)\right]_+^2}{2(\delta \eta T + m/\eta)} - \sum_{t=1}^T \lambda^i_t g_i(\x)\right\}  \\ \leq  \frac{R^2}{2\eta} + \frac{\eta T}{2}\left((m+1)G^2 + 2mD^2)\right).
\end{eqnarray*}
Since $\x_* \in \K$, we have $g_i(\x_*) \leq 0, i \in [m]$, and the resulting inequality becomes
\begin{eqnarray*}
\sum_{t=1}^T f_t(\x_t) - f_t(\x_*) + \sum_{i=1}^m \frac{\Big{[}\sum_{t=1}^T g_i(\x_t)\Big{]}_+^2}{2(\delta \eta T + m/\eta)}  \leq  \frac{R^2}{2\eta} + \frac{\eta T}{2}\left((m+1)G^2 + 2mD^2)\right). 
\end{eqnarray*}
The statement of the first part of the theorem follows by using the expression for $\eta$. The second part is proved by  substituting the regret bound by its lower bound as  $\sum_{t=1}^T f_t(\x_t) - f_t(\x_*) \geq - FT$.
\end{proof}
\begin{remark} We observe that the introduction of quadratic regularizer $\delta\eta \|\blambda\|^2/2$ allows us to turn the expression $\lambda_i \sum_{t=1}^T g_i(\x_t)$ into $\left[\sum_{t=1}^T g_i(\x_t)\right]_+^2$, leading to the bound for the violation of the constraints. In addition, the quadratic regularizer defined in terms of $\blambda$ allows us to work with unbounded $\blambda$ because it cancels  the contribution of the $\|\blambda_t\|$ terms from the loss function and the bound on the gradients $\| \nabla_{\x}\L_t(\x, \blambda)\|$. 
Note that the constraint for $\delta$ mentioned in Theorem  \ref{thm:ogd} is equivalent to
\begin{eqnarray}
    \frac{2}{1/(m+1) + \sqrt{(m+1)^{-2} - 8G^2\eta^2}} \leq \delta \leq \frac{1/(m+1) + \sqrt{(m+1)^{-2} - 8 G^2\eta^2}}{4\eta^2}, \label{eqn:delta-condition}
\end{eqnarray}
from which, when $T$ is large enough (i.e., $\eta$ is small enough), we can simply set $\delta = 2 (m+1) G^2$ that will obey  the constraint in (\ref{eqn:delta-condition}).
\end{remark}
By investigating Lemma \ref{lemma:basic-ineq}, it turns out that the boundedness of the gradients is essential to obtain bounds for Algorithm \ref{alg:1} in Theorem \ref{thm:ogd}. Although, at each iteration, $\blambda_t$ is projected onto the $\R_+^m$,  since $\K$ is a compact set and functions $f_t(\x)$ and $g_i(\x), i \in [m]$ are convex, the boundedness of the functions implies that  the gradients are bounded \citep[Proposition 4.2.3]{convex-analysis-book}.
\subsection {An Efficient Algorithm with $\O(T^{3/4})$ Regret Bound and without Violation of Constraints}
In this subsection we generalize Algorithm \ref{alg:1}  such that the constrained are satisfied in a long run.  To create a sequence of solutions $\{\x_t, t \in [T]\}$ that satisfies the long term constraints $\sum_{t=1}^T g_i(\x_t) \leq 0, i \in [m]$, we make two modifications to Algorithm \ref{alg:1}. First, instead of handling all of the $m$ constraints, we consider a single constraint defined as $g(\x) = \max_{i \in [m]} g_i(\x)$. Apparently,  by achieving zero violation  for the constraint $g(\x) \leq 0 $, it is guaranteed that all of the constraints $g_i(\cdot), i \in [m]$ are also satisfied in the long term. Furthermore,  we change Algorithm \ref{alg:1}   by modifying  the definition of $\L_t(\cdot, \cdot)$ as
\begin{eqnarray}
\label{eqn:alg2}
\L_t(\x, \lambda) = f_t(\x) +  \lambda (g(\x) + \gamma)- \frac{\eta \delta}{2} \lambda^2,
\end{eqnarray}
where $\gamma > 0$ will be decided later. This modification is equivalent to considering the constraint $g(\x) \leq -\gamma$, a tighter constraint than $g(\x) \leq 0$. The main idea behind this modification is that by using a tighter constraint in our algorithm, the resulting sequence of solutions will satisfy the long term constraint $\sum_{t=1}^T g(\x_t) \leq 0$, even though the tighter constraint is violated in many trials.

Before proceeding, we state a fact about the Lipschitz continuity of the function $g(\x)$ in the following proposition.
 \begin{proposition}
 \label{prop:lip}
Assume that functions $g_i(\cdot), i \in [m]$ are Lipschitz continuous with constant G. Then, function $g(\x) = \max_{i \in [m]}{ g_i(\x)}$ is Lipschitz continuous with constant $G$, that is,
\[ |g(\x) - g(\x')| \leq G \| \x - \x'\| \; \; \text{for any} \; \x \in \B \; \text{and}  \; \x' \in \B.\]
\end{proposition}
\begin{proof} First, we  rewrite  $g(\x) = \max_{i \in [m]} g_i(\x)$ as  $g(\x) =  \max_{\alpha \in \Delta_m} \sum_{i=1}^{m}{\alpha_i g_i(\x)}$ where $\Delta_m$ is the $m$-simplex, that is, $\Delta_m= \{ \alpha \in\R_{+}^m; \sum_{i=1}^{m}{\alpha_i} = 1\}$. Then, we have
\begin{eqnarray*}
|g(\x)-g(\x')| &=&  \left | \max_{\alpha \in \Delta_m} \sum_{i=1}^{m}{\alpha_i g_i(\x)} - \max_{\alpha \in \Delta_m} \sum_{i=1}^{m}{\alpha_i g_i(\x')}\right | \\
& \leq& \max_{\alpha \in \Delta_m}  \left | \sum_{i=1}^{m}{\alpha_i g_i(\x)} -  \sum_{i=1}^{m}{\alpha_i g_i(\x')}\right | \\ &\leq& \max_{\alpha \in \Delta_m}  \sum_{i=1}^{m}{\alpha_i \left | g_i(\x) - g_i(\x')\right |}
 \leq G \|\x-\x'\|,
\end{eqnarray*}
where  the last inequality follows from the Lipschitz continuity of $g_i(\x), i \in [m]$.
\end{proof}
To obtain a zero bound on the violation of constraints in the long run,  we make the following assumption about the constraint function $g(\x)$.
\begin{assumption} Let $\K' \subseteq \K$ be the convex set defined as  $\K' = \{\x \in \R^d: g(\x) + \gamma \leq 0\}$ where $\gamma \geq 0$. We assume that the norm of the gradient of the constraint function $g(\x)$ is lower bounded  at the boundary of $\K'$, that is,
\[
    \min\limits_{g(\x) +\gamma = 0} \|\nabla g(\x)\| \geq \sigma.
\]
\label{assumption:1}
\end{assumption}
A direct consequence of {\bf Assumption 1} is that by reducing the domain $\K$ to $\K'$, the optimal value of the constrained optimization problem $\min_{\x \in \K} f(\x)$ does not change much, as revealed by the following theorem. 
\begin{theorem}\label{thm:jin-1} Let $\x_*$ and $\x_\gamma$ be the optimal solutions to the constrained optimization problems defined as $\min_{g(\x) \leq 0} f(\x)$ and $\min_{g(\x) \leq - \gamma} f(\x)$, respectively, where $f(\x) = \sum_{t=1}^{T}{f_t(\x)}$ and $\gamma \geq 0$. We have
\begin{eqnarray*}
    |f(\x_*) - f(\x_{\gamma})| \leq \frac{G}{\sigma}\gamma T.
\end{eqnarray*}
\end{theorem}
\begin{proof}
We note that the  optimization problem $\min_{g(\x) \leq - \gamma} f(\x) = \min_{g(\x) \leq - \gamma} \sum_{t=1}^{T}{f_t(\x)}$, can also be written in the minimax form as
\begin{eqnarray}
f(\x_{\gamma}) = \min \limits_{\x \in \B} \max\limits_{\lambda \in \R_+} \sum_{t=1}^T f_t(\x) +  \lambda (g(\x) + \gamma), \label{eqn:minimax}
\end{eqnarray}
where we use the  fact that $\K' \subseteq \K \subseteq \B$. We denote by $\x_{\gamma}$  and $\lambda_{\gamma}$ the optimal solutions to (\ref{eqn:minimax}). We have
\begin{align*}
f(\x_{\gamma})  & =  \min\limits_{\x \in \B} \max\limits_{\lambda \in \R_+} \sum_{t=1}^T f_t(\x) +  \lambda (g(\x) + \gamma) \\  &=  \min\limits_{\x \in \B} \sum_{t=1}^T f_t(\x) +  \lambda_{\gamma} (g(\x) + \gamma) \nonumber \\
& \leq  \sum_{t=1}^T f_t(\x_*) +  \lambda_{\gamma} (g(\x_*) + \gamma) \nonumber 
 \leq \sum_{t=1}^T f_t(\x_*) +  \lambda_{\gamma} \gamma,
\end{align*}
where the second equality follows the definition of the $\x_{\gamma}$ and the last inequality is due to the optimality of $\x_*$, that is, $g(\x_*) \leq 0$.\\
To bound $|f(\x_{\gamma}) - f(\x_*)|$, we need to bound $ \lambda_{\gamma}$. Since $\x_{\gamma}$ is the minimizer of (\ref{eqn:minimax}), from the optimality condition we have
\begin{eqnarray}
\label{eqn:minusf}
-\sum_{t=1}^T \nabla f_t(\x_{\gamma}) = \lambda_{\gamma} \nabla g(\x_\gamma).
\end{eqnarray}
By setting  $\v = -\sum_{t=1}^T \nabla f_t(\x_{\gamma}) $, we can simplify   (\ref{eqn:minusf}) as  $\lambda_\gamma \nabla g(\x_\gamma) =  \v$. From the KKT optimality condition \citep{boyd-convex-opt}, if $g(\x_\gamma) + \gamma < 0$ then we have $\lambda_\gamma = 0$; otherwise according to Assumption \ref{assumption:1} we can bound $\lambda_\gamma$ by
\begin{eqnarray*}
\label{eqn:14}
\lambda_{\gamma} \leq \frac{\|\v\|}{\| \nabla g(\x_\gamma)\|} \leq  \frac{GT}{\sigma}.
\end{eqnarray*}
We complete the proof by applying the fact $f(\x_*) \leq f(\x_{\gamma}) \leq f(\x_*) + \lambda_{\gamma}\gamma$.
\end{proof}
As indicated by Theorem~\ref{thm:jin-1}, when $\gamma$ is small, we expect the difference between two optimal values $f(\x_*)$ and $f(\x_{\gamma})$ to be small. Using the result from Theorem~\ref{thm:jin-1}, in the following theorem, we show that by running Algorithm \ref{alg:1} on the modified convex-concave functions defined in (\ref{eqn:alg2}), we are able to obtain an $O(T^{3/4})$ regret bound and zero bound on the  violation of  constraints in the long run.

\begin{theorem}
\label{thm:no-violation}
Set  $a = 2R/\sqrt{2G^2+3(D^2+b^2)}$, $\eta = R^2/[a\sqrt{T}]$, and $\delta = 4G^2$. Let $\x_t, t \in [T]$ be the sequence of solutions obtained by Algorithm~\ref{alg:1} with functions defined in (\ref{eqn:alg2}) with $\gamma = b T^{-1/4}$ and $b = 2 \sqrt{F(\delta R^2 a^{-1}+a R^{-2})} $. Let  $\x_*$ be the optimal solution to   $\min_{\x \in \K} \sum_{t=1}^T f_t(\x)$. With sufficiently large $T$, that is, $F T \geq a \sqrt{T} $, and under Assumption \ref{assumption:1}, we have $\x_t, t\in [T]$ satisfy the global constraint $\sum_{t=1}^T g(\x_t) \leq 0$ and the regret $\mathfrak{R}_T$ is bounded by
\begin{eqnarray*}
\mathfrak{R}_T = \sum_{t=1}^T f_t(\x_t) - f_t(\x_*) \leq a \sqrt{T} + \frac{b}{\sigma}GT^{3/4} = \O(T^{3/4}).
\end{eqnarray*}
\end{theorem}
\begin{proof}
Let $\x_{\gamma}$ be the optimal solution to $\min_{g(\x) \leq - \gamma} \sum_{t=1}^T f_t(\x)$. Similar to the proof of Theorem \ref{thm:ogd} when applied to functions in (\ref{eqn:alg2}) we have
\begin{eqnarray*}
\lefteqn{\sum_{t=1}^T f_t(\x_t) - \sum_{t=1}^T f_t(\x) + \lambda \sum_{t=1}^T (g(\x_t)+\gamma)- \left(\sum_{t=1}^T\lambda_t\right)(g(\x) + \gamma) - \frac{\delta\eta T}{2}\lambda^2}  \\
& \leq &-\frac{\delta\eta}{2}\sum_{t=1}^T\lambda_t^2 + \frac{R^2 + \lambda^2}{2\eta} + \frac{\eta T}{2}\left(2G^2+3 (D^2+\gamma^2)\right) + \frac{\eta }{2} \left(2G^2+3\delta^2\eta^2\right)\sum_{t=1}^T \lambda_t^2.
\end{eqnarray*}
By setting  $\delta \geq 2G^2+3\delta^2\eta^2$ which is satisfied by $\delta = 4G^2$, we  cancel the terms including $\lambda_t$ from the right hand side of above inequality.  By maximizing for $\lambda$ over the range $(0, +\infty)$  and noting that $\gamma \leq b$, for the optimal solution $\x_{\gamma}$, we have
\begin{align*}
& \sum_{t=1}^T \left[f_t(\x_t) - f_t(\x_\gamma)\right] + \frac{\Big{[} \sum_{t=1}^T g(\x_t) + \gamma T\Big{]}_{+}^2}{2(\delta \eta T + 1/\eta)}  \leq \frac{R^2}{2\eta} + \frac{\eta T}{2}\left(2G^2 + 3(D^2+b^2)\right),
\end{align*}
which, by optimizing  for $\eta$ and applying the lower bound for the regret as $\sum_{t=1}^T f_t(\x_t) - f_t(\x_\gamma) \geq - FT$, yields the following inequalities 
\begin{eqnarray}
\sum_{t=1}^T f_t(\x_t) - f_t(\x_{\gamma}) \leq a\sqrt{T}  \label{eqn:19} 
\end{eqnarray}
and
\begin{eqnarray}
\sum_{t=1}^T g(\x_t) \leq \sqrt{2\left( FT + a\sqrt{T} \right)\sqrt{T}\left(\frac{\delta R^2}{a} + \frac{ a}{R^2} \right)} - \gamma T, \label{eqn:20}
\end{eqnarray}
for the regret and the violation of the constraint, respectively. Combining (\ref{eqn:19}) with the result of  Theorem~\ref{thm:jin-1} results in
$    \sum_{t=1}^T f_t(\x_\gamma) \leq \sum_{t=1}^T f_t(\x_*) + a \sqrt{T} +(G/\sigma)\gamma T$.
By choosing $\gamma = b T^{-1/4}$ we attain the desired regret bound as  
\[
\sum_{t=1}^T f_t(\x_t) - f_t(\x_*) \leq a \sqrt{T} + \frac{bG}{\sigma}T^{3/4} = O(T^{3/4}).
\]
To obtain the bound on the violation of constraints,  we note that in (\ref{eqn:20}), when $T$ is sufficiently large, that is, $FT \geq a \sqrt{T}$, we have  $\sum_{t=1}^T g(\x_t) \leq 2 \sqrt{F(\delta R^2 a^{-1}+a R^{-2})} T^{3/4} - b T^{3/4}$. Choosing $b = 2 \sqrt{F(\delta R^2 a^{-1}+a R^{-2})} T^{3/4}$ guarantees the zero bound on the violation of constraints as claimed.
\end{proof}

\section{A Mirror Prox Based Approach}
\label{sec:mirror}
The  bound for the violation of constraints for Algorithm~\ref{alg:1} is unsatisfactory since it is significantly worse than $\O(\sqrt{T})$. In this section, we pursue a different approach that
is based on the mirror prox method in \citet{nemirovski-2005-prox} to improve the bound for the violation of constraints. The basic idea is that solving (\ref{eqn:saddle}) can be reduced  to the problem of approximating a saddle point $(\x, \blambda) \in \B \times [0, \infty)^m$ by  solving the associated variational inequality. 

We first define an auxiliary function $\F(\x, \blambda)$ as
\begin{eqnarray*}
    \F(\x, \blambda) = \sum_{i=1}^m \left\{\lambda_i g_i(\x) - \frac{\delta \eta}{2} \lambda_i^2\right\}. \label{eqn:f}
\end{eqnarray*}
In order to  successfully apply the mirror prox method, we follow the fact that  any convex domain can be written as an intersection of linear constraints, and make the following assumption:
\begin{assumption}
\label{assumption:linear}
 We assume that $g_i(\x), i \in [m]$ are linear,  that is, $\K = \{\x \in \R^d: g_i(\x) = \x^{\top}\a_i - b_i\leq 0, i \in [m]\}$ where $\a_i \in \R^d$ is a normalized vector with $\|\a_i\| = 1$  and $b_i \in \R$ . \\
\end{assumption}
The following proposition shows that under Assumptions \ref{assumption:linear}, the function $\F(\x, \blambda)$ has Lipschitz continuous gradient, a basis for the application of the mirror prox method.
\begin{proposition}
\label{prop:prox-lip}
Under Assumption \ref{assumption:linear}, $\F(\x, \blambda)$ has Lipschitz continuous gradient, that is,
\begin{eqnarray*}
{\left\|\nabla_{\x}\F(\x, \blambda) - \nabla_{\x'}\F(\x', \blambda')\right\|^2 + \left\|\nabla_{\blambda}\F(\x, \blambda) - \nabla_{\blambda'} \F(\x', \blambda')\right\|^2}   \leq 2(m+\delta^2\eta^2) ({\|\x - \x'\|^2 + \|\blambda - \blambda'\|^2}).
\end{eqnarray*}
\end{proposition}
\begin{proof}
\begin{eqnarray*}
& & \left\|\nabla_{\x}\F(\x, \blambda) - \nabla_{\x'}\F(\x', \blambda')\right\|^2 + \left\|\nabla_{\blambda}\F(\x, \blambda) - \nabla_{\blambda'} \F(\x', \blambda')\right\|^2 \\
& = & \left\|\sum_{i=1}^m (\lambda_i - \lambda_i') \a_i\right\|^2 + \left\|\sum_{i=1}^m \a_i^{\top}(\x - \x') + \delta \eta \sum_{i=1}^m (\lambda'_i - \lambda_i)\right\|^2 \\
& \leq & \|A^{\top}(\blambda - \blambda')\|^2 + 2\|A(\x - \x')\|^2+ 2\delta^2 \eta^2\|\blambda - \blambda'\|^2 \\
&\leq& 2\sigma_{\max}^2(A)\|\x - \x'\|^2 + (\sigma_{\max}^2(A)+2\delta^2 \eta^2)\|\blambda - \blambda'\|^2.
\end{eqnarray*}
Since
\[ \sigma_{\max}(A) = \sqrt{\lambda_{\max}(AA^{\top})} \leq \sqrt{\text{Tr}(AA^{\top})} \leq \sqrt{m},\]
we have $\sigma_{\max}^2(A) \leq m$, leading to the desired result.
\end{proof}
Algorithm~\ref{alg:2} shows the detailed steps of the mirror prox based algorithm for online convex optimization with long term constraints defined in (\ref{eqn:saddle}). Compared to Algorithm~\ref{alg:1}, there are two key features of Algorithm~\ref{alg:2}. First, it introduces auxiliary variables $\z_t$ and $\bmu_t$ besides the variables $\x_t$ and $\blambda_t$. At each iteration $t$, it first computes the solutions $\x_t$ and $\blambda_t$ based on the auxiliary variables $\z_t$ and $\bmu_t$; it then updates the auxiliary variables based on the gradients computed from $\x_t$ and $\blambda_t$. Second, two different functions are used for updating $(\x_t, \blambda_t)$ and $(\z_t, \bmu_t)$: function $\F(\x, \blambda)$ is used for computing the solutions $\x_t$ and $\blambda_t$, while function $\L_t(\x, \blambda)$ is used for updating the auxiliary variables $\z_t$ and $\bmu_t$. 

\begin{algorithm}[tb]
\caption{Prox Method with Long Term Constraints}
\begin{algorithmic}[1] \label{alg:2}
    \STATE {\bf Input}: constraints $g_i(\x) \leq 0, i \in [m]$, step size $\eta$, and constant $\delta$
    \STATE {\bf Initialization}: $\z_1 = \mathbf{0}$ and $\bmu_1 = \mathbf{0}$
    \FOR{$t = 1, 2, \ldots, T$}
        \STATE Compute the solution for $\x_t$ and $\blambda_t$ as
          \begin{eqnarray*}
            && \x_{t} = \Pi_{\B}\left(\z_t - \eta  \nabla_\x \F(\z_t, \bmu_t)\right)\\ 
            && \blambda_{t} = \Pi_{[0, +\infty)^m}(\bmu_t + \eta \nabla_{\blambda} \F(\z_t, \bmu_t))
           \end{eqnarray*}
        \STATE Submit solution $\x_t$
        \STATE Receive the convex function $f_t(\x)$ and experience loss $f_t(\x_t)$
        \STATE Compute $\L_t(\x, \blambda) = f_t(\x) + \F(\x, \blambda) = f_t(\x) + \sum_{i=1}^m \left\{ \lambda_i g_i(\x) - \frac{\delta \eta}{2} \lambda_i^2\right\} $
        \STATE Update $\z_t$ and $\bmu_t$ by
        \begin{eqnarray*}
        && \z_{t+1} = \Pi_{\B}\left(\z_t - \eta \nabla_\x \L_t(\x_t, \blambda_t)\right)\\ 
        && \bmu_{t+1} = \Pi_{[0, +\infty)^m}(\bmu_t + \eta \nabla_{\blambda} \L_t(\x_t, \blambda_t))
        \end{eqnarray*}
    \ENDFOR
\end{algorithmic}
\end{algorithm}
Our analysis is based on the  Lemma 3.1 from  \citet{nemirovski-2005-prox} which is restated here for completeness. 
\begin{lemma}
\label{lemma:3.1}
Let $\BD(\x, \x')$ be a Bregman distance function that has modulus $\alpha$ with respect to a norm $\|\cdot\|$, that is, $\BD(\x, \x') \geq \alpha\|\x - \x'\|^2/2$. Given $\u \in \B$, $\a$, and $\b$, we set
\[
    \w = \mathop{\arg\min}\limits_{\x \in \B} \a^{\top}(\x - \u) + \BD(\x, \u), \;  \u_+ = \mathop{\arg\min}\limits_{\x \in \B} \b^{\top}(\x - \u) + \BD(\x, \u).
\]
Then for any $\x \in \B$ and $\eta > 0$, we have
\[
\eta \b^{\top}(\w - \x) \leq \BD(\x, \u) - \BD(\x, \u_+) + \frac{\eta^2}{2\alpha} \|\a - \b\|_*^2 - \frac{\alpha}{2}\left[ \|\w - \u\|^2 + \|\w - \u_+\|^2\right].
\]
\end{lemma}
We equip  $\B \times [0, +\infty)^m$ with the norm $\|\cdot\|$ defined as
\[
\|(\z, \bmu)\|^2 = \frac{\|\z\|^2 + \|\bmu\|^2}{2},
\]
where $\|\cdot\|^2$ is the Euclidean norm defined separately for each domain. It is immediately seen that the Bregman distance function defined as
\[
\BD(\z_t, \bmu_t, \z_{t+1}, \bmu_{t+1}) = \frac{1}{2}\|\z_t - \z_{t+1}\|^2 + \frac{1}{2}\|\bmu_t- \bmu_{t+1}\|^2
\]
 is $\alpha = 1$ modules with respect to the norm $\|\cdot\|$.
 
To analyze the mirror prox  algorithm, we begin with a simple lemma which is the direct application of Lemma \ref{lemma:3.1} when applied to the updating rules of Algorithm \ref{alg:3}.
 \begin{lemma}
 \label{lemma:prox-lemma}
 If $\eta(m+\delta^2 \eta^2) \leq \frac{1}{4}$ holds, we have
\begin{align*}
& \L_t(\x_t, \blambda) - \L_t(\x, \blambda_t) \\ & \leq \frac{\|\x - \z_t\|^2 - \|\x - \z_{t+1}\|^2}{2\eta} + \frac{\|\blambda- \bmu_t\|^2 - \| \blambda- \bmu_{t+1}\|^2}{2\eta} + \eta\|\nabla f_t(\x_t)\|^2.
\end{align*}
\end{lemma}
\begin{proof}
To apply Lemma \ref{lemma:3.1}, we define $\u$, $\w$, $\u_+$, $\a$ and $\b$ as follows
\begin{eqnarray*}
& & \u = (\z_t, \bmu_t), \u_+ = (\z_{t+1}, \bmu_{t+1}), \w = (\x_t, \blambda_t), \\
& & \a = (\nabla_\x \F(\z_t, \bmu_t), - \nabla_{\blambda} \F(\z_t, \bmu_t)), \b = (\nabla_\x \L_t(\x_t, \blambda_t), - \nabla_{\blambda} \L_t(\x_t, \blambda_t)).
\end{eqnarray*}
Using Lemmas \ref{lemma:basic-ineq} and  \ref{lemma:3.1}, we have

\begin{eqnarray*}
\lefteqn{\L_t(\x_t, \blambda) - \L_t(\x, \blambda_t) - \frac{\|\x - \z_t\|^2 - \|\x - \z_{t+1}\|^2}{2\eta} -\frac{\|\blambda-\bmu_t\|^2 - \|\blambda-\bmu_{t+1}\|^2}{2\eta}} \\
& \leq & \underbrace { \frac{\eta}{2}\left\{ \left\|\nabla_{\x} \F(\z_t, \bmu_t) - \nabla_\x \L_t(\x_t, \blambda_t)\right\|^2 + \left\|\nabla_{\blambda}\F(\z_t, \bmu_t) - \nabla_{\blambda} \L_t(\x_t, \blambda_t)\right\|^2\right\} }_{\textbf{I}} \\ & \hspace{5mm}- & \underbrace { \frac{1}{2}\left\{\|\z_t - \x_t\|^2 + \|\bmu_t - \blambda_t\|^2 \right\}}_{\textbf{II}}. \\
\end{eqnarray*}
By expanding the gradient terms and applying the inequality $(a+b)^2 \leq 2(a^2+b^2)$, we upper bound (\textbf{I}) as:

\begin{align}
\label{eqn:a}
 (\textbf{I})  & = \frac{\eta}{2}  \{ 2 \| \nabla f_t(\x_t)\|^2 + 2  \|\nabla_{\x} \F(\z_t, \bmu_t) - \nabla_\x \F(\x_t, \blambda_t)\|^2 + \|\nabla_{\blambda}\F(\z_t, \bmu_t) - \nabla_{\blambda} \F(\x_t, \blambda_t)\|^2 \}  \nonumber \\ 
& \leq  \eta  \| \nabla f_t(\x_t) \|^2 + \eta \left\{ \| \nabla_{\x} \F(\z_t, \bmu_t)-\nabla_\x \F(\x_t, \blambda_t)\|^2  +  \|\nabla_{\blambda} \F(\x_t, \blambda_t)- \nabla_{\blambda} \F(\x_t, \blambda_t)\|^2 \right\} \nonumber \\
& \leq  \eta  \| \nabla f_t(\x_t)\|^2 +  2 \eta (m+\delta^2 \eta^2) \left \{ \| \z_t - \x_t \|^2 + \| \bmu_t - \blambda_t \|^2 \right \},
\end{align}
where  the last  inequality follows from Proposition \ref{prop:prox-lip}.
Combining (\textbf{II}) with (\ref{eqn:a}) results in
\begin{eqnarray*}
\lefteqn{\L_t(\x_t, \blambda) - \L_t(\x, \blambda_t) - \frac{\|\x - \z_t\|^2 - \|\x - \z_{t+1}\|^2}{2\eta} -\frac{\|\blambda-\bmu_t\|^2 - \|\blambda-\bmu_{t+1}\|^2}{2\eta}} \\
& \leq & \eta \|\nabla f_t(\x_t)\|_2^2 + \left(2\eta(m+\delta^2 \eta^2) - \frac{1}{2}\right)\left\{\|\z_t - \x_t\|^2 + \|\bmu_t - \blambda_t\|_2^2 \right\}.
\end{eqnarray*}
We complete the proof by rearranging the terms and setting $\eta(m+\delta^2 \eta^2) \leq \frac{1}{4}$. 
\end{proof}
\begin{theorem}
\label{thm:12}
Set $\eta = T^{-1/3}$ and $\delta = T^{-2/3}$. Let $\x_t, t \in [T]$ be the sequence of solutions obtained by Algorithm~\ref{alg:2}.  Then for $T \geq 164(m+1)^3$ we have
\begin{eqnarray*}
\sum_{t=1}^T f_t(\x_t) - f_t(\x_*) \leq  \O(T^{2/3})\;\; \text{and}\;\;\sum_{t=1}^T g_i(\x_t) \leq \O(T^{2/3}).
\end{eqnarray*}
\end{theorem}
\begin{proof}
Similar to the proof of Theorem \ref{thm:ogd}, by summing the bound in Lemma \ref{lemma:prox-lemma} for all rounds $t = 1, \cdots, T$, and taking maximization for $\blambda$ we have the following inequality  for any $\x_* \in \K$,

\[
\sum_{t=1}^T \left[f_t(\x_t) - f_t(\x_*)\right] + \sum_{i=1}^m \frac{\left[\sum_{t=1}^T g_i(\x_t)\right]_+^2}{2(\delta \eta T + m/\eta)}  \leq  \frac{R^2}{2\eta} + \frac{\eta T}{2}G^2.
\]
By setting $\delta = \frac{1}{\eta T}$ and using the fact  that $\sum_{t=1}^T f_t(\x_t) - f_t(\x_*) \geq - FT$ we have:
\[ \sum_{t=1}^T \left[f_t(\x_t) - f_t(\x)\right] \leq \frac{R^2}{2\eta} + \frac{\eta T}{2}G^2 \]
and
\[ \sum_{t=1}^{T}{g_i(\x_t)}  \leq \sqrt{(1+\frac{m}{\eta})  \left( \frac{R^2}{\eta} + \eta T G^2 + FT \right )} . \]
Substituting the stated value for  $\eta$,  we get the desired bounds as mentioned in the theorem. Note that the condition $\eta(m+\delta^2 \eta^2) \leq \frac{1}{4}$ in Lemma \ref{lemma:prox-lemma} is satisfied for the stated values of $\eta$ and $\delta$ as long as $T \geq 164(m+1)^3$.
\end{proof}
Using the same trick as Theorem \ref{thm:no-violation}, by introducing appropriate $\gamma$, we will be able to establish the solutions that exactly satisfy the  constraints in the long run with an $\O(T^{2/3})$ regret bound as shown in the following corollary. In the case when all the constraints are linear, that is, $g_i(\x) = \a_i^{\top}\x \leq b_i, i \in [m]$, Assumption 1 is simplified into the following condition,
\begin{eqnarray}
\label{eqn:assumption}
    \min\limits_{\alpha \in \Delta_m} \left\| \sum_{i=1}^m \alpha_i \a_i\right\| \geq \sigma, 
 \end{eqnarray}
where $\Delta_m$ is a $m$ dimensional simplex, that is, $\Delta_m =\{ \alpha \in \R_+^m: \sum_{i=1}^m \alpha_i = 1\}$. This is because $g(\x) = \max_{\alpha \in \Delta_m} \sum_{i=1}^m \alpha_ig_i(\x)$ and as a result, the (sub)gradient of $g(\x)$ can always be written as
$ \partial g(\x) = \sum_{i=1}^m \alpha_i \nabla g_i(\x) = \sum_{i=1}^m \alpha_i \a_i$
where $\alpha \in \Delta_m$.  As an illustrative example, consider the case when the norm vectors $\a_i, i \in [m]$ are linearly independent. In this case the condition mentioned in (\ref{eqn:assumption}) obviously holds which indicates that the assumption does not limit the applicability of the proposed algorithm. 

\begin{corollary}
\label{corollary:prox}
Let  $\eta = \delta = T^{-1/3}$. Let $\x_t, t \in [T]$ be the sequence of solutions obtained by Algorithm~\ref{alg:2}  with $\gamma = b T^{-1/3}$ and $b = 2\sqrt{F}$. With sufficiently large $T$, that is, $FT \geq R^2T^{1/3} + G^2 T^{2/3}$, under Assumptions \ref{assumption:linear} and condition in (\ref{eqn:assumption}), we have $\x_t, t\in [T]$ satisfy the global constraints $\sum_{t=1}^T g_i(\x_t) \leq 0, i \in [m]$ and the regret $\mathfrak{R}_T$ is bounded by
\begin{eqnarray*}
\mathfrak{R}_T = \sum_{t=1}^T f_t(\x_t) - f_t(\x_*) \leq \frac{R^2}{2}T^{1/3}+ \left(\frac{G^2}{2}+ \frac{2G\sqrt{F}}{\sigma}\right)T^{2/3} = \O(T^{2/3}).
\end{eqnarray*}
\end{corollary}
The proof is similar to that of Theorem \ref{thm:no-violation} and we defer it to Appendix \ref{app:corollary:prox}. As indicated by Corollary \ref{corollary:prox}, for any convex domain defined by a  finite number of  halfspaces, that is, Polyhedral set, one can easily replace the projection onto the Polyhedral set with the ball containing the Polyhedral at the price of  satisfying the constraints in the long run and achieving $\O(T^{2/3})$ regret bound. 
\section{Online Convex Optimization with Long Term Constraints under Bandit Feedback for Domain}
\label{sec:bandit}
We now turn to extending  the gradient based convex-concave optimization algorithm discussed in Section \ref{sec:ogd}  to the setting where the learner only receives partial feedback for constraints. More specifically, the  exact definition of  the domain $\K$ is not  exposed to the learner, only that  the solution is within a ball $\B$. Instead, after receiving a solution $\x_t$, the oracle will present the learner with the convex loss function $f_t(\x)$ and the maximum  violation of the constraints  for $\x_t$, that is, $g(\x_t) = \max_{i \in [m]}{ g_i(\x_t)}$.  We remind that the function $g(\x)$ defined in this way is  Lipschitz continuous with constant $G$ as proved  in Proposition~\ref{prop:lip}. In this setting, the convex-concave function defined in (\ref{eqn:L}) becomes as 
\[\L_t(\x, \lambda) = f_t(\x)+ \lambda g(\x)- (\delta \eta /2) \lambda^2.\]

  The mentioned  setting is closely tied to the bandit online  convex optimization. In the bandit setting, in contrast to the full information setting, only the cost of the chosen decision (i.e.,  the incurred loss $f_t(\x_t)$) is revealed to the algorithm, not the function itself.  There is a rich body of literature that deals with the bandit online convex optimization.  In  the seminal papers of \citet{DBLP:conf/soda/FlaxmanKM05} and \citet{DBLP:conf/stoc/AwerbuchK04} it has been shown that one could design algorithms with $\O(T^{3/4})$ regret bound even in the bandit setting where only evaluations of the loss functions are revealed at a single point.  If we specialize to the online bandit optimization of linear loss functions, \citet{DBLP:conf/nips/DaniHK07} proposed an inefficient algorithm with $\O(\sqrt{T})$ regret bound and \citet{DBLP:conf/colt/AbernethyHR08} obtained $\O(\sqrt{T \log T})$ bound by an efficient algorithm if the convex set admits an efficiently computable self-concordant barrier. For general convex loss functions, \citet{DBLP:conf/colt/AgarwalDX10} proposed optimal algorithms in a new bandit setting, in which multiple points can be queried for the cost values.  By using multiple evaluations, they showed that the modified online gradient descent algorithm can achieve $\O(\sqrt{T})$ regret bound in expectation.

\begin{algorithm}[tb]
\caption{ Multipoint Bandit Online Convex Optimization with Long Term Constraints} 

\begin{algorithmic}[1] \label{alg:3}
    \STATE {\bf Input}: constraint $g(\x)$, step size $\eta$, constant $\delta > 0$, exploration parameter $\zeta > 0$, and shrinkage coefficient $\xi$
    \STATE {\bf Initialization}: $\x_1 = \mathbf{0}$ and $\lambda_1 = {0}$
    \FOR{$t = 1, 2, \ldots, T$}
        \STATE Submit solution $\x_t$
        \STATE Select  unit  vector $\u_t$ uniformly  at random
        \STATE Query $g(\x)$ at points $\x_t+\zeta \u_t$ and $\x_t-\zeta \u_t$ and incur average of them as violation of constraints
        \STATE Compute $\tilde{g}_{\x,t} = \nabla f_t(\x_t)+ \lambda_t \left[\frac{d}{2 \zeta} (g(\x_t+\zeta \u_t) - g(\x_t-\zeta \u_t)) \u_t \right] $

	\STATE Compute $\tilde{g}_{\lambda,t} =  \frac{1}{2} (g(\x_t+\zeta \u_t) + g(\x_t-\zeta \u_t)) - \eta \delta \lambda_t$

        \STATE Receive the convex function $f_t(\x)$ and experience loss $f_t(\x_t)$
        \STATE Update $\x_t$ and $\lambda_t$ by
        \begin{eqnarray*}
        && \x_{t+1}= \Pi_{(1-\xi)\B} \left(\x_t - \eta \tilde{g}_{\x,t} \right) \\ 
        && \lambda_{t+1} = \Pi_{[0, +\infty)}(\lambda_t + \eta \tilde{g}_{\lambda,t}) \quad
        \end{eqnarray*}
    \ENDFOR
\end{algorithmic}
\end{algorithm}

Algorithm \ref{alg:3} gives a complete description of the proposed algorithm under the bandit setting, which is a slight
modification of Algorithm \ref{alg:1}. Algorithm \ref{alg:3} accesses the constraint function $g(\x)$ at {two points}. To facilitate the analysis, we define
\[  \widehat{\L}_t(\x, \lambda) = f_t(\x) + \lambda \hat{g}(\x) - \frac{\eta \delta}{2} \lambda^2,\]
where $\hat{g}(\x)$ is the smoothed version of $g(\x)$ defined as $\hat{g}(\x) = \mathbb{E}_{\v \in \mathbb{S}} {[} \frac{d}{\zeta} g(\x + \zeta \v) \v{]}$ at point $\x_t$ where $\mathbb{S}$ denotes the  unit sphere centered at the origin. Note that $\hat{g}(\x)$ is Lipschitz continuous with the same constant $G$, and it is always differentiable even though $g(\x)$ is not in our case.

Since we do not have access to  the function $\hat{g}(\cdot)$ to compute $\nabla_{\x}\widehat{\L}(\x, \lambda)$,  we need a way to estimate its gradient  at point $\x_t$. Our gradient estimation  closely follows  the idea in \citet{DBLP:conf/colt/AgarwalDX10} by querying  $g(\x)$ function at two points. The main advantage of using two points to estimate the gradient with respect to one point gradient estimation used in \citet{DBLP:conf/soda/FlaxmanKM05} is that the former has a bounded norm which is independent of $\zeta$ and leads to improved regret bounds.

The gradient estimators for $\nabla_{\x}\widehat{\L}(\x_t, \lambda_t) = \nabla f(\x_t)+ \lambda_t \nabla\hat{g}(\x_t)$ and   $\nabla_{\lambda}\widehat{\L}(\x_t, \lambda_t) = \hat{g}(\x_t) - \delta \eta \lambda_t$  in Algorithm \ref{alg:3} are computed by evaluating the $g(\x)$ function at two random points  around $\x_t$ as
\begin{eqnarray}
\tilde{g}_{\x,t} = \nabla f_t(\x_t)+ \lambda_t \Bigg{[}\frac{d}{2 \zeta} (g(\x_t+\zeta \u_t) - g(\x_t-\zeta \u_t)) \u_t \Bigg{]} \nonumber
\end{eqnarray}
 and
\begin{eqnarray}
\tilde{g}_{\lambda,t} =  \frac{1}{2} (g(\x_t+\zeta \u_t) + g(\x_t-\zeta \u_t)) - \eta \delta \lambda_t,
\nonumber
\end{eqnarray}
where  $\u_t$ is chosen uniformly  at random from the surface of the unit sphere. Using Stock's theorem, \citet{DBLP:conf/soda/FlaxmanKM05} showed that $\frac{1}{2\zeta} (g(\x_t+\zeta \u_t) - g(\x_t-\zeta \u_t)) \u_t$ is a conditionally unbiased estimate of the gradient of  $\hat{g}(\x)$ at point $\x_t$. To make sure that randomized points around $\x_t$ live inside the convex domain $\B$, we need to stay away from the boundary of the set such that the ball of radius $\zeta$ around $\x_t$ is contained in $\B$. In particular, in \citet{DBLP:conf/soda/FlaxmanKM05} it has been shown that for any $\x \in (1 - \xi) \B$ and any unit vector $\u$ it holds that $(\x + \zeta \u) \in \B$ as soon as $\zeta \in [0, \xi r]$.

In order to facilitate the analysis of the Algorithm \ref{alg:3}, we define the convex-concave  function $\H_t(\cdot, \cdot)$ as
\begin{eqnarray}
\label{eqn:h}
 \H_t(\x, \lambda) = \widehat{\L}_t(\x, \lambda) + \left( \tilde{g}_{\x,t} - \nabla_{\x} \widehat{\L}(\x_t, \lambda_t)\right)\x+\left( \tilde{g}_{\lambda,t} - \nabla_{\lambda} \widehat{\L}(\x_t, \lambda_t)\right)\lambda.
\end{eqnarray}
It is easy to check that  $\nabla_{\x} \H(\x_t,  \lambda_t) = \tilde{g}_{\x,t} $ and $\nabla_{\lambda} \H(\x_t, \lambda_t) = \tilde{g}_{\lambda,t} $. By defining functions $\H_t(\x, \lambda)$, Algorithm \ref{alg:3} reduces to Algorithm \ref{alg:1} by doing gradient descent on functions  $\H_t(\x, \lambda)$ except the projection is made onto  the set $(1-\xi) \B$ instead of $\B$.   \\

We begin our analysis by reproducing Proposition \ref{prop:2} for functions $\H_t(\cdot, \cdot)$.
\begin{lemma}
\label{lemma:h}
If the Algorithm \ref{alg:1} is performed over convex set $\K$ with functions ${\H}_t$ defined in (\ref{eqn:h}), then for any $\x \in \K$ we have
\begin{align*} 
\sum_{t=1}^T \H_t(\x_t, \lambda) - \H_t(\x, \lambda_t)  \leq \frac{R^2 + \|\lambda\|_2^2}{2\eta}+  \eta (D^2+G^2) T + \eta (d^2 G^2+\eta^2 \delta^2) \sum_{t=1}^{T}{\lambda_t^2}.
 \end{align*}
\end{lemma}
\begin{proof}
We have  $\nabla_{\x}\H_t(\x_t, \lambda_t) = \tilde{g}_{\x,t}$  and $\nabla_{\lambda}\H_t(\x_t, \lambda_t) = \tilde{g}_{\lambda,t}$. It is straightforward to show that  $\frac{1}{2\zeta} (g(\x_t+\zeta \u_t) - g(\x_t-\zeta \u_t)) \u_t$ has norm bounded by $Gd$ \citep{DBLP:conf/colt/AgarwalDX10}.
So, the norm of  gradients are bounded as   $\|\tilde{g}_{\x,t}\|_2^2 \leq 2(G^2+ d^2 G^2 \lambda_t^2)$ and $\|\tilde{g}_{\lambda,t}\|_2^2 \leq 2(D^2+\eta^2 \delta^2 \lambda_t^2)$.  Using Lemma \ref{lemma:basic-ineq}, by adding
for all rounds we get the desired inequality.
\end{proof}
The following theorem gives the regret bound and the expected violation of the constraints in the long run for Algorithm \ref{alg:3}.
\begin{theorem}
\label{thm:bandit}
Let $ c =  \sqrt{D^2+G^2}(\sqrt{2}R+ \frac{\sqrt{2}D}{\delta R}) + (\frac{D}{r} + 1)\frac{GD}{r} $.
Set $\eta = R/\sqrt{ 2(D^2+G^2)T}$. Choose $\delta$ such that $\delta \geq 2 (d^2 G^2+\eta^2 \delta^2)$. Let $\zeta = \frac{\delta}{T}$ and $\xi = \frac{\zeta}{r}$. Let $\x_t, t \in [T]$ be the sequence of solutions obtained by Algorithm~\ref{alg:3}. We then have
\begin{eqnarray*}
&& \sum_{t=1}^{T}{  f_t(\x_t) - f_t(\x)} \leq \frac{GD}{r} + c \sqrt{T} = \O(T^{1/2}), \;\text{and}\\
&& \E \Big{[}\sum_{t=1}^{T}{ g(\x_t)} \Big{]}\leq   G \delta+\sqrt{  \Big{(}\frac{\delta R^2+2(D^2+G^2)}{R \sqrt{D^2+G^2}}\Big{)}(\frac{GD}{r} + c \sqrt{T} + F T) \sqrt{T}} = \O(T^{3/4}).
\end{eqnarray*}
\end{theorem}

\begin{proof}
Using Lemma \ref{lemma:basic-ineq} for the functions $\widehat{\L}_t(\cdot, \cdot)$ and $\H_t(\cdot, \cdot)$ we have 
$$\widehat{\L}_t(\x_t, \lambda) - \widehat{\L}_t(\x, \lambda_t) \leq (\x_{t} - \x)^{\top}\nabla_\x \widehat{\L}_t(\x_t, \lambda_t) - (\lambda - \lambda_{t})^{\top}\nabla_{\lambda} \widehat{\L}_t(\x_t, \lambda_t),$$ 
and also 
$$\H_t(\x_t, \lambda) - \H_t(\x, \lambda_t) \leq (\x_{t} - \x)^{\top}\tilde{g}_{\x,t} - (\lambda - \lambda_{t})^{\top} \tilde{g}_{\lambda,t}.$$
Subtracting the preceding inequalities, taking expectation, and  summing   for all $t$ from $1$ to $T$ we get
\begin{align}
\label{eqn:landh}
& \E \Bigg{[}\sum_{t=1}^T \widehat{\L}_t(\x_t, \lambda) - \widehat{\L}_t(\x, \lambda_t) \Bigg{]} \nonumber\\
& = \E \Bigg{[}\sum_{t=1}^T \H_t(\x_t, \lambda)-\H_t(\x, \lambda_t)  \Bigg{]}\\ 
& \hspace{5mm} + \E \Bigg{[}\sum_{t=1}^{T}{(\x_{t} - \x)^{\top}(\nabla_\x \widehat{\L}_t(\x_t, \lambda_t) - \E_t[\tilde{g}_{\x_t,t}]) + (\lambda_{t} - \lambda)^{\top}(\nabla_\lambda \widehat{\L}_t(\x_t, \lambda_t) -  \E_t[\tilde{g}_{\lambda_t,t}])\Bigg{]}}. \nonumber
\end{align}
Next we provide an upper bound on the difference between the gradients of two functions. First,  $ \E_t [ \tilde{g}_{\x,t} ]= \nabla_\x \widehat{\L}_t(\x_t, \lambda_t)$, so  $\tilde{g}_{\x, t}$ is an unbiased estimator of $\nabla_{\x} \widehat{\L}_t(\x_t, \lambda_t)$. Considering the update rule for $ \lambda_{t+1}$ we have $|\lambda_{t+1} | \leq (1-\eta^2 \delta) |\lambda_{t}| + \eta D$ which implies that  $ |\lambda_t| \leq \frac{D}{\delta \eta}$ for all $t$. So we obtain
\begin{align}
\label{eqn:lambdaub}
 & (\lambda_{t} - \lambda)^{\top}(\nabla_\lambda \widehat{\L}_t(\x_t, \lambda_t) - \E_t[\tilde{g}_{\lambda_t,t}])  \nonumber\\
& \leq | \lambda_{t} - \lambda | \E_t \Big{[}\|\nabla_\lambda \widehat{\L}_t(\x_t, \lambda_t) - \tilde{g}_{\lambda_t,t} \|_2\Big{]} \nonumber \\
& \leq \frac{D}{\delta \eta}\left | \frac{1}{2} (g(\x_t+\zeta \u_t) + g(\x_t-\zeta \u_t)) - \hat{g}(\x_t)\right | \leq \frac{DG}{\delta \eta}\zeta \|\u_t\| \leq \frac{DG}{\delta \eta} \zeta,
\end{align}
where the last inequality follows from Lipschitz property of the functions $g(\x)$ and $\hat{g}(\x)$ with the same  constant $G$.
Combining the inequalities (\ref{eqn:landh}) and (\ref{eqn:lambdaub})  and using Lemma \ref{lemma:h}, we have
\begin{align*}
 \E \Big{[}\sum_{t=1}^T \widehat{\L}_t(\x_t, \lambda) - \widehat{\L}_t(\x, \lambda_t)\Big{]} \leq \frac{R^2 + \lambda^2}{2\eta}+  \eta (D^2+G^2) T + \eta (d^2 G^2+\eta^2 \delta^2) \sum_{t=1}^{T}{\lambda_t^2} + \frac{D G\zeta}{\delta \eta} T.
\end{align*}
By expanding  the right hand side of above inequality, we obtain
\begin{align*}
\label{eqn:18}
& \sum_{t=1}^T \left[f_t(\x_t) - f_t((1 - \xi)\x)\right] + \lambda \mathbb{E} \Big{[}\sum_{t=1}^{T}{\hat{g}(\x_t)\Big{]} - \E\Big{[}\hat{g}((1 - \xi)\x)} \Big{]}\sum_{t=1}^{T}{\lambda_t} - \frac{\eta \delta T}{2} \lambda^2 + \frac{\eta \delta}{2}\sum_{t=1}^{T}{\lambda_t^2} \\
& \leq   \frac{R^2 + \lambda^2}{2\eta}+  \eta (D^2+G^2) T + \eta (d^2 G^2+\eta^2 \delta^2) \sum_{t=1}^{T}{\lambda_t^2} + \frac{DG \zeta}{\delta \eta} T. \nonumber
\end{align*}
By choosing $ \delta \geq 2 (d^2 G^2+\eta^2 \delta^2) $ we  cancel $\lambda_t^2$ terms from both sides and have
\begin{align}
& \sum_{t=1}^T \left[f_t(\x_t) - f_t((1 - \xi)\x)\right] + \lambda \mathbb{E} \Big{[}\sum_{t=1}^{T}{\hat{g}(\x_t) \Big{]}- \E\Big{[}\hat{g}((1 - \xi)\x)} \Big{]}\sum_{t=1}^{T}{\lambda_t} - \frac{\eta \delta T}{2} \lambda^2  \nonumber \\
& \leq   \frac{R^2 + \lambda^2}{2\eta}+  \eta (D^2+G^2) T +  \frac{D G\zeta}{\delta \eta} T.
\end{align}
By convexity  and Lipschitz property of $f_t(\x)$ and $g(\x)$ we have
\begin{equation}
\label{eqn:fminusx}
f_t((1-\xi)\x) \leq (1-\xi)f_t(\x) + \xi f_t(0) \leq f_t(\x)  +   DG \xi, 
\end{equation}
\begin{equation}
\label{eqn:gminusx}
g(\x) \leq  \hat{g}(\x)+ G \zeta\;, \;\text{and} \; \; \; \hat{g}((1-\xi)\x) \leq  g((1-\xi)\x)+ G \zeta \leq g(\x) + G \zeta+ D G\xi.
\end{equation}
Plugging (\ref{eqn:fminusx}) and (\ref{eqn:gminusx}) back into  (\ref{eqn:18}), for any optimal solution $\x_* \in \K$ we get
\begin{align}
\label{eqn:22}
& \sum_{t=1}^T \left[f_t(\x_t) - f_t(\x)\right] + \lambda \mathbb{E}\Big{[} \sum_{t=1}^{T}{g(\x_t)} \Big{]} - \frac{\eta \delta T}{2} \lambda^2  - \lambda G \zeta T \nonumber \\
& \leq  \frac{R^2 + \lambda^2}{2\eta}+  \eta (D^2+G^2) T +  \frac{D G\zeta}{\delta \eta} T +  DG \xi T + (DG\xi + G \zeta) \sum_{t=1}^{T}{\lambda_t}.
\end{align}
Considering the fact that $\lambda_t \leq \frac{D}{\delta \eta}$ we have $\sum_{t=1}^{T}{\lambda_t} \leq \frac{DT}{\delta \eta}$. Plugging back into the (\ref{eqn:22}) and rearranging the terms we have
\begin{align*}
& \sum_{t=1}^T \left[f_t(\x_t) - f_t(\x)\right] + \lambda \mathbb{E} \Big{[}\sum_{t=1}^{T}{g(\x_t)} \Big{]} - \frac{\eta \delta T}{2} \lambda^2  - \lambda G \zeta T - \frac{\lambda^2}{2 \eta}\\
& \leq  \frac{R^2}{2\eta}+  \eta (D^2+G^2) T +  \frac{DG\zeta}{\delta \eta} T +  DG \xi T + (DG\xi + G \zeta)  \frac{DT}{\delta \eta}. \nonumber
\end{align*}
By setting $\xi = \frac{\zeta}{r}$ and $\zeta = \frac{1}{T}$ we get
\begin{align}
 \sum_{t=1}^T \left[f_t(\x_t) - f_t(\x)\right]  \leq  \frac{R^2}{2\eta}+  \eta (D^2+G^2) T +  \frac{DG \zeta}{\delta \eta} T +  \frac{\zeta DG  T}{r}+ (\frac{D}{r} + 1) \zeta \frac{ DGT}{\delta \eta}, \nonumber
\end{align}
which gives the mentioned regret bound by optimizing for $\eta$. Maximizing for $\lambda$ over the range $(0, +\infty)$ and using $\sum_{t=1}^T f_t(\x_t) - f_t(\x_*) \geq - FT$, yields the following inequality for the violation of  constraints
\begin{eqnarray*}
\frac{\Bigg{[}\E\left[\sum_{t=1}^T g(\x_t)\right] - G \zeta T \Bigg{]}_+^2}{4(\delta \eta T/2+1/2\eta)} \leq \frac{DG}{r} + c \sqrt{T} + F T.
\end{eqnarray*}
Plugging in the stated values of parameters completes the proof.
Note that $\delta = 4 d^2 G^2$ obeys  the condition specified in the theorem.
\end{proof}
\section{Conclusion}
\label{sec:con}
In this study we have addressed the problem of online convex optimization with constraints, where we only need the constraints to be satisfied in the long run. In addition to the regret bound which is the main tool in analyzing the performance of  general online convex optimization algorithms, we defined the bound on the violation of constraints in the long term which measures the cumulative violation of the solutions from the constraints for all rounds. Our setting is applied to solving online convex optimization without projecting the solutions onto the complex convex domain at each iteration,  which may be computationally inefficient for complex domains.  Our strategy is to turn the problem into an online convex-concave optimization problem and apply online gradient descent algorithm to solve it. We have proposed efficient algorithms in three different settings; the violation of constraints is allowed, the constraints need to be exactly satisfied, and finally  we do not have access to the target convex domain except it is bounded by a ball.  Moreover, for domains determined by linear constraints, we used the  mirror prox method, a simple gradient based algorithm for variational inequalities, and obtained  an $\O(T^{2/3})$ bound for both regret and the violation of the constraints.

Our work leaves open a number of interesting directions for future work. In particular it would be interesting to see if it is possible to improve the bounds obtained in this paper, i.e.,  getting an $\O(\sqrt{T})$ bound on the regret and  better bound than $\O(T^{3/4})$  on the violation of constraints for general convex domains. Proving optimal lower bounds for the proposed setting also remains  as an open question. Also, it would be interesting to consider strongly convex loss  or constraint functions. Finally, relaxing the assumption we made to exactly satisfy the constraints in the long run is  an interesting problem  to be investigated.
\acks{The authors would like to thank the Action Editor and three anonymous reviewers for their constructive
comments and helpful suggestions on the original version of this paper. This work was supported in part by National Science Foundation (IIS-0643494) and Office of Navy Research (Award N000141210431).}





\appendix

\section{Proof of Theorem~\ref{thm:a}}
\label{app:thm:a}
We first show that when $\delta < 1$, there exists a loss function and a constraint function such that the violation of constraint is linear in $T$. To see this, we set $f_t(\x) = \w^{\top}\x, t \in [T]$ and $g(\x) = 1 - \w^{\top}\x$. Assume we start with an infeasible solution, that is, $g(\x_1) > 0$ or $\x_1^{\top}\w < 1$. Given the solution $\x_t$ obtained at $t$th trial, using the standard gradient descent approach, we have $\x_{t+1} = \x_t - \eta (1 - \delta)\w$. Hence, if $\x_t^{\top}\w < 1$, since we have $\x_{t+1}^{\top}\w < \x_t^{\top}\w < 1$,  if we start with an infeasible solution, all the solutions obtained over the trails will violate the constraint $g(\x) \leq 0$, leading to a linear number of violation of constraints. Based on this analysis, we assume $\delta > 1$ in the analysis below.

Given a strongly convex loss function $f(\x)$ with modulus $\gamma$, we consider a constrained optimization problem given by
\begin{eqnarray*}
    \min\limits_{g(\x) \leq 0} f(\x), \label{eqn:constrain-opt}
\end{eqnarray*}
which is equivalent to the following unconstrained optimization problem
\[
    \min\limits_{\x} f(\x) + \lambda [g(\x)]_+,
\]
where $\lambda \geq 0$ is the Lagrangian multiplier. Since we can always scale $f(\x)$ to make $\lambda \leq 1/2$, it is safe to assume $\lambda \leq 1/2 < \delta$. Let $\x_*$ and $\x_a$ be the optimal solutions to the constrained optimization problems $\arg\min_{g(\x) \leq 0} f(\x)$ and $\arg\min\limits_{\x} f(\x) + \delta [g(\x)]_+$, respectively. We choose $f(\x)$ such that $\|\nabla f(\x_*)\| > 0$, which leads to $\x_a \neq \x_*$. This holds because according to the first order optimality condition, we have
\[
    \nabla f(\x_*) = -\lambda\nabla g(\x_*), \; \nabla f(\x_a) = -\delta \nabla g(\x_*),
\]
and therefore $\nabla f(\x_*) \neq \nabla f(\x_a)$ when $\lambda < \delta$. Define $\Delta = f(\x_a) - f(\x_*)$. Since $\Delta \geq \gamma\|\x_a - \x_*\|^2/2$ due to the strong convexity of $f(\x)$, we have $\Delta > 0$.

Let $\{\x_t\}_{t=1}^T$ be the sequence of solutions generated by the OGD algorithm that minimizes the modified loss function $f(\x) + \delta [g(\x)]_+$. We have
\begin{eqnarray*}
&&\sum_{t=1}^T f(\x_t) + \delta[g(\x_t)]_+ \geq T\min\limits_{\x} f(\x) + \delta [g(\x)]_+ \\
& = & T(f(\x_a) + \delta[g(\x_a)]_+) \geq T(f(\x_a) + \lambda[g(\x_a)]_+) \\
& = & T(f(\x_*) + \lambda[g(\x_*)]_+) + T(f(\x_a) + \lambda[g(\x_a)]_+ - f(\x_*) - \lambda [g(\x_*)]) \\
& \geq & T\min\limits_{g(\x) \leq 0} f(\x) + T\Delta.
\end{eqnarray*}
As a result, we have
\[
    \sum_{t=1}^T f(\x_t) + \delta[g(\x_t)]_+ - \min\limits_{g(\x) \leq 0} f(\x) = O(T),
\]
implying that either the regret $\sum_{t=1}^T f(\x_t) - Tf(\x_*)$ or the violation of the constraints $\sum_{t=1}^T [g(\x)]_+$ is linear in $T$.

To better understand the performance of penalty based approach, here we analyze the performance of the OGD in solving the  online optimization problem in (\ref{eqn:setting1}). The algorithm is analyzed using the following lemma from \citet{DBLP:conf/icml/Zinkevich03}.

\begin {lemma} 
\label{lemma:basic-ogd}
Let $\x_1, \x_2, \ldots, \x_T$ be the sequence of solutions obtained by applying OGD on the sequence of bounded convex functions $f_1, f_2, \ldots, f_T$. Then, for any  solution $\x_* \in \K$ we have
\[ \sum_{t=1}^{T}{f_t(\x_t)} - \sum_{t=1}^{T}{f_t(\x_*)} \leq \frac{R^2}{2\eta} + \frac{\eta}{2} \sum_{t=1}^{T}{\| \nabla f_t(\x_t)\|^2}. \]
\end{lemma}

We apply OGD to  functions $\fh_t(\x), \; t \in [T]$ defined in (\ref{eqn:fhat}), that is, instead of updating the solution based on the gradient of $f_t(\x)$, we update the solution by the gradient of $\fh_t(\x)$. Using Lemma \ref{lemma:basic-ogd}, by expanding the functions $\fh_t(\x)$ based on (\ref{eqn:fhat}) and considering the fact that $\sum_{i=1}^{m}{[g_i(\x_*)]_{+}^2} = 0$, we get
\begin{eqnarray}
\label{eqn:basic-inq}
\sum_{t=1}^{T}{f_t(\x_t)} - \sum_{t=1}^{T}{f_t(\x_*)} + \frac{\delta}{2} \sum_{t=1}^{T}{\sum_{i=1}^{m}{[g_i(\x)]_{+}^2}} \leq \frac{R^2}{2\eta} + \frac{\eta}{2} \sum_{t=1}^{T}{\| \nabla \fh_t(\x_t)\|^2}. 
\end{eqnarray}
From the definition of $\fh_t(\x)$, the norm of the gradient $ \nabla \fh_t(\x_t)$ is bounded as follows
\begin{eqnarray}
\label{eqn:gupper}
\| \nabla \fh_t(\x)\|^2 = \| \nabla f_t(\x) + \delta \sum_{i = 1}^{m}{[g_i(\x)]_+ \nabla g_i(\x)}\|^2  
  \leq  2 G^2 (1+m\delta^2 D^2),
\end{eqnarray}
where the inequality holds because $(a_1+a_2)^2 \leq 2(a_1^2+a_2^2)$. By substituting (\ref{eqn:gupper}) into the (\ref{eqn:basic-inq}) we have:
\begin{eqnarray}
\label{eqn:5}
\sum_{t=1}^{T}{f_t(\x_t)} - \sum_{t=1}^{T}{f_t(\x_*)} + \frac{\delta}{2} \sum_{t=1}^{T}{\sum_{i=1}^{m}{[g_i(\x_t)]_{+}^2}} \leq \frac{R^2}{2\eta} + \eta G^2 (1+m\delta^2 D^2) T.
\end{eqnarray}
Since $[\cdot]_{+}^2$ is a convex function, from Jensen's inequality and following the fact that 
$\sum_{t=1}^T f_t(\x_t) - f_t(\x_*) \geq - FT$, we have:
 \begin{eqnarray*}
\frac{\delta}{2T}  \sum_{i=1}^{m}{ \left [ \sum_{t=1}^{T}{g_i(\x_t)} \right]_{+}^2 } \leq \frac{\delta}{2}\sum_{i=1}^{m}{ \sum_{t=1}^{T}{[g_i(\x_t)]_{+}^2}}  \leq \frac{R^2}{2\eta} + \eta G^2 (1+m\delta^2 D^2) T + FT.
 \end{eqnarray*}
By minimizing the right hand side of (\ref{eqn:5}) with respect to $\eta$, we get the regret bound as
\begin{eqnarray}
\label{eqn:penalty-regret}
 \sum_{t=1}^{T}{f_t(\x_t)} - \sum_{t=1}^{T}{f_t(\x_*)} \leq RG \sqrt{2(1+m\delta^2D^2)T} = \O(\delta \sqrt{T})
 \end{eqnarray}
and the bound for the violation of constraints as
\begin{eqnarray}
\label{eqn:violation-bound}
\sum_{t=1}^{T}{g_i(\x_t)}  \leq \sqrt{ \left( \frac{R^2}{2\eta} + \eta G^2 (1+m\delta^2 D^2) T + FT \right) \frac{2T}{ \delta} } = \O (T^{1/4} \delta^{1/2} + T \delta^{-1/2}).
 \end{eqnarray}
Examining the bounds obtained in (\ref{eqn:penalty-regret}) and (\ref{eqn:violation-bound}), it turns out that in order to recover $\O(\sqrt{T})$ regret bound, we need to set $\delta$ to be a constant, leading to  $\O(T)$ bound for the violation of constraints in the long run, which is not satisfactory at all. 
\section{Proof of Corollary~\ref{corollary:prox}}
\label{app:corollary:prox}
Let $\x_{\gamma}$ be the optimal solution to $\min_{g(\x) \leq - \gamma} \sum_{t=1}^T f_t(\x)$. Similar to the proof of Theorem \ref{thm:12}, we have
\begin{eqnarray*}
\sum_{t=1}^T \left[f_t(\x_t) - f_t(\x_\gamma)\right] + \frac{\left[ \sum_{t=1}^T g(\x_t) + \gamma T\right]_{+}^2}{2(\delta \eta T + 1/\eta)} \leq \frac{R^2}{2\eta}+\frac{\eta T}{2}G^2. 
\end{eqnarray*}
Using the stated values for the parameters $\eta = \delta = T^{-1/3}$, and applying the fact that $\sum_{t=1}^T f_t(\x_t) - f_t(\x_\gamma) \geq - FT$ we obtain,
\begin{eqnarray}
\label{app:eqn:1}
 \sum_{t=1}^T f_t(\x_t) - f_t(\x_\gamma)  \leq \frac{R^2}{2}T^{1/3} + \frac{G^2}{2}T^{2/3}
 \end{eqnarray}
and
\begin{eqnarray}
\label{app:eqn:2}
\left[\sum_{t=1}^T g(\x_t) + \gamma T\right]_{+}^2 \leq 2\big{(}R^2T^{1/3} + G^2 T^{2/3}+ FT\big{)} T^{1/3}.
\end{eqnarray}
From Theorem~\ref{thm:jin-1}, we have the bound
\begin{eqnarray}
\label{app:eqn:3}
    \sum_{t=1}^T f_t(\x_\gamma) \leq \sum_{t=1}^T f_t(\x_*) + \frac{G}{\sigma}\gamma T.
\end{eqnarray}
Combining inequalities (\ref{app:eqn:1}) and (\ref{app:eqn:3})  with substituting the stated value of $\gamma = b T^{-1/3}$ yields the regret bound as desired. To obtain the bound for the violation of the constraints, from (\ref{app:eqn:2}) we have
\begin{eqnarray*}
\label{app:eqn:4}
 \sum_{t=1}^T g(\x_t) \leq  \sqrt{2\big{(}R^2T^{1/3} + G^2 T^{2/3}+ FT\big{)} T^{1/3}}  - b T^{2/3}. 
 \end{eqnarray*}
For sufficiently large values of $T$, that is, $FT \geq R^2T^{1/3} + G^2 T^{2/3}$ we can simplify above inequality as $ \sum_{t=1}^T g(\x_t) \leq 2 \sqrt{F} T^{2/3} - b T^{2/3}$. By setting $b = 2 \sqrt{F}$ the zero bound on the violation of constraints is guaranteed.
\bibliography{online-long}

\begin{thebibliography}{24}
\providecommand{\natexlab}[1]{#1}
\providecommand{\url}[1]{\texttt{#1}}
\expandafter\ifx\csname urlstyle\endcsname\relax
  \providecommand{\doi}[1]{doi: #1}\else
  \providecommand{\doi}{doi: \begingroup \urlstyle{rm}\Url}\fi

\bibitem[Abernethy et~al.(2008)Abernethy, Hazan, and
  Rakhlin]{DBLP:conf/colt/AbernethyHR08}
Jacob Abernethy, Elad Hazan, and Alexander Rakhlin.
\newblock Competing in the dark: An efficient algorithm for bandit linear
  optimization.
\newblock In \emph{COLT}, pages 263--274, 2008.

\bibitem[Abernethy et~al.(2009)Abernethy, Agarwal, Bartlett, and
  Rakhlin]{DBLP:conf/colt/AbernethyABR09}
Jacob Abernethy, Alekh Agarwal, Peter~L. Bartlett, and Alexander Rakhlin.
\newblock A stochastic view of optimal regret through minimax duality.
\newblock In \emph{COLT}, 2009.

\bibitem[Abernethy et~al.(2012)Abernethy, Hazan, and
  Rakhlin]{DBLP:journals/tit/AbernethyHR12}
Jacob Abernethy, Elad Hazan, and Alexander Rakhlin.
\newblock Interior-point methods for full-information and bandit online
  learning.
\newblock \emph{IEEE Transactions on Information Theory}, 58\penalty0
  (7):\penalty0 4164--4175, 2012.

\bibitem[Agarwal et~al.(2010)Agarwal, Dekel, and
  Xiao]{DBLP:conf/colt/AgarwalDX10}
Alekh Agarwal, Ofer Dekel, and Lin Xiao.
\newblock Optimal algorithms for online convex optimization with multi-point
  bandit feedback.
\newblock In \emph{COLT}, pages 28--40, 2010.

\bibitem[Awerbuch and Kleinberg(2004)]{DBLP:conf/stoc/AwerbuchK04}
Baruch Awerbuch and Robert~D. Kleinberg.
\newblock Adaptive routing with end-to-end feedback: distributed learning and
  geometric approaches.
\newblock In \emph{STOC}, pages 45--53, 2004.

\bibitem[Bartlett et~al.(2007)Bartlett, Hazan, and
  Rakhlin]{DBLP:conf/nips/BartlettHR07}
Peter~L. Bartlett, Elad Hazan, and Alexander Rakhlin.
\newblock Adaptive online gradient descent.
\newblock In \emph{NIPS}, pages 257--269, 2007.

\bibitem[Bernstein et~al.(2010)Bernstein, Mannor, and
  Shimkin]{DBLP:conf/nips/BernsteinMS10}
Andrey Bernstein, Shie Mannor, and Nahum Shimkin.
\newblock Online classification with specificity constraints.
\newblock In \emph{NIPS}, pages 190--198, 2010.

\bibitem[Bertsekas et~al.(2003)Bertsekas, Nedic, and
  Ozdaglar]{convex-analysis-book}
Dimitri~P. Bertsekas, Angelia Nedic, and Asuman~E. Ozdaglar.
\newblock \emph{Convex Analysis and Optimization}.
\newblock Athena Scientific, 2003.

\bibitem[Boyd and Vandenberghe(2004)]{boyd-convex-opt}
Stephen Boyd and Lieven Vandenberghe.
\newblock \emph{Convex Optimization}.
\newblock Cambridge University Press, 2004.

\bibitem[Cesa-Bianchi and Lugosi(2006)]{Cesa-Bianchi:2006:PLG:1137817}
Nicolo Cesa-Bianchi and Gabor Lugosi.
\newblock \emph{Prediction, Learning, and Games}.
\newblock Cambridge University Press, 2006.

\bibitem[Cesa-Bianchi et~al.(2004)Cesa-Bianchi, Conconi, and
  Gentile]{DBLP:journals/tit/Cesa-BianchiCG04}
Nicol{\`o} Cesa-Bianchi, Alex Conconi, and Claudio Gentile.
\newblock On the generalization ability of on-line learning algorithms.
\newblock \emph{IEEE Transactions on Information Theory}, 50\penalty0
  (9):\penalty0 2050--2057, 2004.

\bibitem[Dani et~al.(2007)Dani, Hayes, and Kakade]{DBLP:conf/nips/DaniHK07}
Varsha Dani, Thomas~P. Hayes, and Sham Kakade.
\newblock The price of bandit information for online optimization.
\newblock In \emph{NIPS}, 2007.

\bibitem[Duchi et~al.(2008)Duchi, Shalev-Shwartz, Singer, and
  Chandra]{DBLP:conf/icml/DuchiSSC08}
John Duchi, Shai Shalev-Shwartz, Yoram Singer, and Tushar Chandra.
\newblock Efficient projections onto the {\it }$l_1$-ball for learning in high
  dimensions.
\newblock In \emph{ICML}, pages 272--279, 2008.

\bibitem[Flaxman et~al.(2005)Flaxman, Kalai, and
  McMahan]{DBLP:conf/soda/FlaxmanKM05}
Abraham Flaxman, Adam~Tauman Kalai, and H.~Brendan McMahan.
\newblock Online convex optimization in the bandit setting: gradient descent
  without a gradient.
\newblock In \emph{SODA}, pages 385--394, 2005.

\bibitem[Hazan et~al.(2007)Hazan, Agarwal, and
  Kale]{Hazan:2007:LRA:1296038.1296051}
Elad Hazan, Amit Agarwal, and Satyen Kale.
\newblock Logarithmic regret algorithms for online convex optimization.
\newblock \emph{Machine Learning}, 69\penalty0 (2-3):\penalty0 169--192, 2007.

\bibitem[Liu and Ye(2009)]{DBLP:conf/icml/LiuY09}
Jun Liu and Jieping Ye.
\newblock Efficient euclidean projections in linear time.
\newblock In \emph{ICML}, pages 83--90, 2009.

\bibitem[Mannor and Tsitsiklis(2006)]{DBLP:conf/colt/MannorT06}
Shie Mannor and John~N. Tsitsiklis.
\newblock Online learning with constraints.
\newblock In \emph{COLT}, pages 529--543, 2006.

\bibitem[Mannor et~al.(2009)Mannor, Tsitsiklis, and
  Yu]{DBLP:journals/jmlr/MannorTY09}
Shie Mannor, John~N. Tsitsiklis, and Jia~Yuan Yu.
\newblock Online learning with sample path constraints.
\newblock \emph{Journal of Machine Learning Research}, 10:\penalty0 569--590,
  2009.

\bibitem[Narayanan and Rakhlin(2010)]{DBLP:conf/nips/NarayananR10}
Hariharan Narayanan and Alexander Rakhlin.
\newblock Random walk approach to regret minimization.
\newblock In \emph{NIPS}, pages 1777--1785, 2010.

\bibitem[Nemirovski(1994)]{nemirovski-efficient}
Arkadi Nemirovski.
\newblock Efficient methods in convex programming.
\newblock Lecture Notes, Available at http://www2.isye.gatech.edu/~nemirovs,
  1994.

\bibitem[Nemirovski(2005)]{nemirovski-2005-prox}
Arkadi Nemirovski.
\newblock Prox-method with rate of convergence o(1/t) for variational
  inequalities with lipschitz continuous monotone operators and smooth
  convex-concave saddle point problems.
\newblock \emph{SIAM J. on Optimization}, 15\penalty0 (1):\penalty0 229--251,
  2005.

\bibitem[Rakhlin(2009)]{lecture-notes}
Alexander Rakhlin.
\newblock Lecture notes on online learning.
\newblock Lecture Notes, Available at
  http://www-stat.wharton.upenn.edu/~rakhlin/papers, 2009.

\bibitem[Shalev-Shwartz and Kakade(2008)]{DBLP:conf/nips/Shalev-ShwartzK08}
Shai Shalev-Shwartz and Sham~M. Kakade.
\newblock Mind the duality gap: Logarithmic regret algorithms for online
  optimization.
\newblock In \emph{NIPS}, pages 1457--1464, 2008.

\bibitem[Zinkevich(2003)]{DBLP:conf/icml/Zinkevich03}
Martin Zinkevich.
\newblock Online convex programming and generalized infinitesimal gradient
  ascent.
\newblock In \emph{ICML}, pages 928--936, 2003.

\end{thebibliography}

\end{document}